\documentclass[letterpaper]{article} 
\usepackage{aaai25}  
\usepackage{times}  
\usepackage{helvet}  
\usepackage{courier}  
\usepackage[hyphens]{url}  
\usepackage{graphicx} 
\usepackage{natbib}  
\usepackage{caption} 
\usepackage[dvipsnames]{xcolor}
\usepackage{amsfonts}       
\usepackage{amsmath}
\usepackage{amssymb} 
\usepackage{mathtools} 
\usepackage{mathrsfs}
\usepackage{nicefrac}       
\usepackage{bbold} 

\usepackage{enumitem}
\usepackage{subcaption} 
\usepackage{booktabs}       

\usepackage{stmaryrd} 
\usepackage{amsthm}

\usepackage{thmtools, thm-restate} 
\usepackage{mdframed} 
\usepackage{multirow} 

\usepackage{nccmath}

\usepackage{cleveref} 

\usepackage[linesnumbered,ruled,vlined]{algorithm2e}
\usepackage{algorithmic}    

\SetKwInput{KwInput}{Input}               
\SetKwInput{KwInitialization}{Initialization} 
\SetKwInput{KwOutput}{Output}

\usepackage{pifont} 
\newcommand{\cmark}{\ding{51}}%
\newcommand{\xmark}{\ding{55}}%

\usepackage{makecell, multirow}

\usepackage{soul} 

\definecolor{taborange}{HTML}{FF7F0e}
\definecolor{tabred}{HTML}{d62728}
\definecolor{tabgreen}{HTML}{2ca02c}
\newcommand{\tcmv}[1]{\textcolor{tabgreen}{#1}}
\newcommand{\tcmo}[1]{\textcolor{taborange}{#1}}
\newcommand{\tcmr}[1]{\textcolor{tabred}{#1}}

\newtheorem*{theorem*}{Theorem}
 
\newtheorem{theorem}{Theorem}
 
\newtheorem{proposition}{Proposition} 
\newtheorem{remark}{Remark}
\newtheorem{corollary}{Corollary}

\newtheorem{property}{Property}

\newtheorem{takeaway}{Take-away}
\newtheorem{drawback}{Drawback}

\crefname{algo}{Algorithm}{Algorithms}
\crefname{model}{Model}{Models}
\crefname{lemma}{Lemma}{Lemmas}
\crefname{fact}{Fact}{Facts}
\crefname{theorem}{Theorem}{Theorems}
\crefname{corollary}{Corollary}{Corollaries}
\crefname{claim}{Claim}{Claims}
\crefname{example}{Example}{Examples}
\crefname{problem}{Problem}{Problems}
\crefname{definition}{Definition}{Definitions}
\crefname{assumption}{Assumption}{Assumptions}

\crefname{subsection}{Subsection}{Subsections}
\crefname{section}{Section}{Sections}
\crefname{algorithm}{Algorithm}{Algorithms}
\crefname{algocf}{alg.}{algs.}
\Crefname{algocf}{Algorithm}{Algorithms}
\crefname{proposition}{Proposition}{Propositions}
\crefname{exemple}{Exemple}{Examples}
\crefname{remark}{Remark}{Remarks}
\crefname{property}{Property}{Properties}
\crefname{inequality}{Inequality}{Inequalities}

\DeclareMathOperator{\e}{e}

\newcommand{\ffrac}[2]{\ensuremath{\frac{\displaystyle #1}{\displaystyle #2}}}

\renewcommand{\SS}{\mathbf{S}}
\newcommand{\SSigma}{\mathbf{\Sigma}}
\newcommand{\UU}{\mathbf{U}}
\newcommand{\VV}{\mathbf{V}}
\newcommand{\EE}{\mathbf{E}}

\renewcommand{\AA}{\mathbf{A}}
\newcommand{\BB}{\mathbf{B}}

\newcommand{\XX}{\mathbf{X}}
\newcommand{\QQ}{\mathbf{Q}}

\newcommand{\PP}{\mathbf{P}}

\newcommand{\qqquad}{\qquad\qquad}

\newcommand{\E}{\mathbb{E}}
\newcommand{\N}{\mathbb{N}}
\newcommand{\R}{\mathbb{R}}

\newcommand{\lnrm}{\left \|} 
\newcommand{\rnrm}{\right \|} 

\newcommand{\vertiii}[1]{{\vert\kern-0.25ex\vert\kern-0.25ex\vert #1 
    \vert\kern-0.25ex\vert\kern-0.25ex\vert}}

\newcommand{\Proba}[1]{\mathbb{P} \left[#1\right]} 

\newcommand{\FullExpec}[1]{\E \left[#1\right]} 
\newcommand{\fullexpec}[1]{\E [#1]} 

\newcommand{\pdtscl}[2]{\langle#1,#2\rangle}  
\newcommand{\SqrdNrm}[1]{ \lnrm #1\rnrm^2} 
\newcommand{\sqrdnrm}[1]{ \| #1\|^2} 
\newcommand{\bigpar}[1]{\left( #1 \right)} 

\newcommand{\Tr}[1]{\mathrm{Tr}\bigpar{#1}}

\newcommand{\Diag}[1]{\mathrm{Diag}\bigpar{#1}}

\newcommand{\Id}{\mathbf{I}}
\newcommand{\rank}[1]{\mathrm{rank}\bigpar{#1}}

\newcommand{\frob}{\mathrm{F}}

\newcommand{\iN}{{i=1}^N}
\newcommand{\OneToN}{\{1, \dots, N\}}

\title{In-depth Analysis of Low-rank Matrix Factorisation in a Federated Setting}
\author{
Constantin Philippenko, Kevin Scaman, Laurent Massoulié \\
}

\affiliations{
    Inria Paris - Département d’informatique de l’ENS, PSL Research University \\
    firstname.lastname@inria.fr
}

\begin{document}

\maketitle

\begin{abstract}
    We analyze a distributed algorithm to compute a low-rank matrix factorization on $N$ clients, each holding a local dataset $\mathbf{S}^i \in \mathbb{R}^{n_i \times d}$, mathematically, we seek to solve $min_{\mathbf{U}^i \in \mathbb{R}^{n_i\times r}, \mathbf{V}\in \mathbb{R}^{d \times r} } \frac{1}{2} \sum_{i=1}^N \|\mathbf{S}^i - \mathbf{U}^i \mathbf{V}^\top\|^2_{\text{F}}$. Considering a power initialization of $\mathbf{V}$, we rewrite the previous smooth non-convex problem into a smooth strongly-convex problem that we solve using a parallel Nesterov gradient descent potentially requiring a single step of communication at the initialization step. For any client $i$ in $\{1, \dots, N\}$, we obtain a global $\mathbf{V}$ in $\mathbb{R}^{d \times r}$ common to all clients and a local variable $\mathbf{U}^i$ in $\mathbb{R}^{n_i \times r}$. We provide a linear rate of convergence of the excess loss which depends on $\sigma_{\max} / \sigma_{r}$, where $\sigma_{r}$ is the $r^{\mathrm{th}}$ singular value of the concatenation $\mathbf{S}$ of the matrices $(\mathbf{S}^i)_{i=1}^N$. This result improves the rates of convergence given in the literature, which depend on $\sigma_{\max}^2 / \sigma_{\min}^2$. We provide an upper bound on the Frobenius-norm error of reconstruction under the power initialization strategy. We complete our analysis with experiments on both synthetic and real data.
\end{abstract}

\begin{links}
    \link{Code}{https://github.com/philipco/matrix_factorization}
\end{links}

\section{Notation}

For $\mathbf{A}$ a matrix in $\R^{n \times d}$, we note $\rank{\mathbf{A}}$ its rank, $\lambda_1(\mathbf{A}) \geq \dots \geq \lambda_{\rank{\mathbf{A}}}(\mathbf{A})$ are its decreasing eigenvalues, $\sigma_1(\mathbf{A}) \geq \dots \geq \sigma_{\rank{\mathbf{A}}}(\mathbf{A}) \geq 0$ are its positive and decreasing singular values. More specifically, we note $\lambda_{\max}(\mathbf{A}),\lambda_{\min}(\mathbf{A})$ (resp. $\sigma_{\max}(\mathbf{A}),\sigma_{\min}(\mathbf{A})$) the biggest/smallest eigenvalue (resp. singular value) of $\mathbf{A}$. We define the condition number of $\mathbf{A}$ as $\kappa(\mathbf{A}) := \sigma_{\max}(\mathbf{A}) / \sigma_{\min}(\mathbf{A})$. Furthermore, we note $\sqrdnrm{\mathbf{A}}_\frob := \sqrt{\Tr{\mathbf{A} \mathbf{A}^\top}} = (\sum_{i=1}^{\min\{n,d\}} \sigma_i^2(\mathbf{A}))^{1/2}$ the Frobenius norm and $\|\mathbf{A}\|_2 := \sqrt{\lambda_{\max}(\mathbf{A}^\top \mathbf{A})} = \sigma_{\max}(\mathbf{A})$ the operator norm induced by the $2$-norm. 
The group of orthogonal matrices is denoted $\mathcal{O}_d(\R)$. For $a,b \in \R$, we also denote $a \wedge b := \min(a,b)$.

\section{Introduction}

The problem of low-rank matrix factorization is widely analyzed in machine learning \citep[e.g.][]{deshpande2006adaptive,achlioptas2007fast,liberty2007randomized,nguyen2009fast,rokhlin2010randomized,halko2011finding,witten2015randomized,tropp2019streaming,tropp2023randomized}. Indeed, several key challenges can be reduced to it, for instance: clustering \citep{li2006relationships}, features learning \citep{bisot2017feature}, dictionary learning \citep{mensch2016dictionary}, anomaly detection \citep{tong2011non}, denoising \citep{wilson2008speech}, or matrix completion \citep{jain2013low}. Let $\SS$ in $\R^{n \times d}$ and $r \in \N^*$, the study of the low-rank matrix factorization (MF) problem corresponds to find $\UU \in \R^{n\times r}$ and $ \VV \in \R^{d \times r}$ minimizing:
\begin{align}
\tag{MF}
\label{eq:low_rank_factorization}
    \frac{1}{2} \SqrdNrm{\SS - \UU \VV^\top}_\frob := F(\UU, \VV) \leq \epsilon_{\min} + \epsilon \,,
\end{align}
where $\epsilon_{\min}$ is the minimal achievable error w.r.t. the Forbenius-norm, $\epsilon$ is the error induced by the algorithm, and $r$ denote the \emph{latent dimension}. 
The \citet{eckart1936approximation} theorem shows that $\sum_{i >r} \sigma_i^2$ (we omit to indicate the matrix $\SS$ for its eigen/singular values) is the minimal Frobenius-norm error
when approximating $\SS$ with a rank-$r$ matrix, and it is $\sigma_{r+1}$ for the $2$-norm \citep{mirsky1960symmetric}.
\begin{proposition}
\label{prop:lower_bound_with_proba}
    We have for the Frobenius-norm $\min_{\UU, \VV \in \R^{n \times r} \times \R^{d \times r}} \SqrdNrm{\SS - \UU \VV^\top}_\frob = \sum_{i>r} \sigma_i^2$ and for the $2$-norm $\min_{\UU, \VV \in \R^{n \times r} \times \R^{d \times r}}  \|\SS - \UU \VV^\top \|_2 = \sigma_{r+1}$.
\end{proposition}

\Cref{prop:lower_bound_with_proba} corresponds to the thin SVD decomposition \citep{eckart1936approximation} and to keeping the $r$ largest singular values, the error is thus determined by the $n \wedge d - r$ smallest singular values of the spectrum of $\SS$.

Contemporary machine learning problems frequently involve data spread across multiple clients, each holding a subset of the dataset, thereby introducing an additional layer of complexity.
In this paper, we consider the federated setting \citep{konecny_federated_2016,mcmahan_communication-efficient_2017} where a client $i$ holds a dataset $\SS^i$ in $\R^{n_i \times d}$ with $n_i$ in $\N$ rows and $d$ constant features across the $N$ in $\N$ clients. We want to factorize for any client $i$ its dataset based on a global shared variable $\VV$ in $\R^{d \times r}$ and a local personalized variable $\UU^i$ in $\R^{n_i \times r}$, \Cref{eq:low_rank_factorization} can thus be rewritten as minimizing:
\begin{align}
\tag{Dist. MF}
\label{eq:dist_MF}
    \frac{1}{2} \sum_\iN \SqrdNrm{\SS^i - \UU^i \VV^\top}_\frob := \sum_\iN F^i(\UU^i, \VV)\,. 
\end{align} 

We note $\SS$ and $\UU$ the vertical concatenation of matrices $(\SS^i, \UU^i)_\iN$. We write the SVD of $\SS = \UU_* \mathbf{\SSigma} \VV_*^\top$ where $\UU_*$ in $\R^{n \times n}$ ($n = \sum_\iN n_i$) is the left basis vectors s.t. $\UU_*^\top \UU_* = \Id_r$,  $\VV_*$ in $\R^{d \times d}$ is the right basis vectors s.t. $\VV_*^\top \VV_* = \Id_r$, and  $\mathbf{\SSigma}= \Diag{\lambda_1(\SS), \dots, \lambda_{n \wedge d}(\SS)}$ in $\R^{n \times d}$ contains the singular value of $\SS$. 
We consider that there exists a true low-rank matrix $\XX$ in $\R^{n \times d}$ and a white noise $\EE$ in $\R^{n \times d}$ s.t. $\SS = \XX + \EE$, which implies $r_* = \rank{\XX} \leq \rank{\SS} \leq \min(n, d)$. The $r_*$ first eigenvalues of $\SS$ correspond to the signal (potentially noised if $\EE \neq 0$)  and the $n \wedge d - r_*$ last correspond to the white noise from $\EE$. 

\begin{remark}
    In distributed settings, the errors $\epsilon_{\min}$ and $\epsilon$ are determined by the spectrum of $\SS$, and not by the spectrum of clients' matrices $(\SS_i)_{i=1}^N$.
\end{remark}

Our work progresses in two directions: (1) offering new theoretical insights on matrix factorization and (2) developing a robust algorithm suited for federated environments, the challenge being to design a distributed algorithm that computes a low-rank matrix factorization of $\SS$ with minimal communication/computation overhead, while guaranteeing a linear convergence towards the reconstruction error~$\epsilon_{\min} + \epsilon$. 
To tackle this challenge, we combine a global distributed randomized power iteration \citep{halko2011finding} with local gradient descents on each client, bridging these two lines of research.

The randomized power iteration method \citep{halko2011finding,hardt2014noisy,li2021communication,saha2023matrix} is one of the main approaches used in practice\footnote{For instance, the \texttt{TruncatedSVD} class of Scikit-learn \citep{scikit-learn} or the \texttt{svd\_lowrank} function of PyTorch \citep{paszke2019pytorch} are based on the work of \citet{halko2011finding}.}  to compute low-rank matrix factorizations, therefore.
Theoretically, the authors have provided in a \emph{centralized} setting, solely \emph{if taking  $r \leq r_* +2$}, a bound on the \emph{$2$-norm for $\alpha$ in $\N$} \emph{proportional to $r$}. In this work, we are interested in giving \textit{new results in the \textbf{federated setting} on the \textbf{Frobenius-norm} and \textbf{$r$ in $\{1, \dots, n_i \wedge d\}$}}.  Indeed, using the existing results on the $2$-norm to derive results on the Frobenius norm would result in an undesirable $\sqrt{d}$-factor to both $\epsilon_{\min}$ and $\epsilon$. Thus, the results on the two different norms are not directly comparable and ask for a new~analysis.

On the other hand, literature based on gradient descent \citep{jain2017global,zhu2019distributed,ye2021global,jain2013low,ward2023convergence,gu2023low} usually provides results on the Frobenius norm with linear convergence rates depending on $\kappa^{2}(\SS)$ \citep{ward2023convergence}, which might be arbitrarily large, thus hindering the convergence. The primary challenge with this approach lies in the non-convexity of $F$, rendering the solution space for \Cref{eq:low_rank_factorization} infinite and impeding its theoretical analysis. Numerous research endeavors are directed towards these algorithms with the aims of (1) enhancing the theoretical convergence rate, (2) writing elegant and concise proofs, and (3) understand the underlying mechanism that allows for convergence. More generally, advancements in understanding \Cref{eq:low_rank_factorization} can significantly enrich our comprehension of other non-convex problems such as matrix sensing or deep learning optimization \citep{du2018algorithmic}.

In \Cref{sec:related_work}, we conduct a short overview of the related work.
In \Cref{sec:algo_gd}, we propose and analyze a parallel algorithm derived from \citet{halko2011finding}, we explain the associated computational/communicational trade-offs, exhibit a linear rate of convergence, and provide an upper bound in probability of both the condition number $\kappa(\SS)$ and the Frobenius-norm error.  In \Cref{sec:experiments} we illustrate our theoretical findings with experiments on both synthetic and real data, and finally, in \Cref{sec:conclusion}, we conclude by highlighting the main takeaways of this article and by listing a few open directions.

\paragraph{Contributions.} We make the following contributions.
\begin{itemize}
    \item We solve the low-rank matrix factorization problem with a local personalized variable $\UU^i$ in $\R^{n_i \times r}$ and a global shared $\VV$ in $\R^{r \times r}$ s.t. $\SS^i \approx \UU^i \VV^\top$ for any client $i$ in $\OneToN$. Following \citet{halko2011finding}, we use a power method to initialize $\VV$, then we run a gradient descent on each client to compute~$\UU^i$.
    \item We provide a novel result regarding the Frobenius-norm of $\mathbf{S} - \mathbf{UV}^\top$. Our result is non-trivial as it was lacking in \citet[][see Remark 10.1]{halko2011finding}. This result provides new insight on \Cref{eq:low_rank_factorization} and can not be compared for $\alpha \geq 1$ to \citet{halko2011finding,hardt2014noisy} or \citet{li2021communication} which analyze different quantities.
    \item In the distributed setting, under low noise conditions, unlike \citet{li2021communication}, we can achieve a finite number of communications in $\Omega\bigpar{\frac{\log( \sigma_{\max} d r_*^2  \epsilon^{-1})}{\log(\sigma_{r_*}) - \log(\sigma_{r_*+1})}}$. And potentially a single communication stage is enough, i.e., our algorithm is possibly \textit{embarrassingly parallel}.
    \item Algorithmically, our method involves sampling different random Gaussian matrices $\Phi$ to obtain better condition numbers, this sampling allows rapid gradient descent, independently of $\sigma_{\min}$. It increases the average number of exchanges by a factor $m$, yet it ensures convergence almost surely. We are the first to propose such a result for~\eqref{eq:low_rank_factorization} problems.
    \item Compared to existing literature on (distributed) gradient descent \citep[e.g.][]{zhu2019distributed,ward2023convergence} our approach surpasses them both in terms of communicational and computational complexity. Our guarantee of convergence are stronger and our analysis bridges the gap between works on distributed gradient descent and on the power method. Besides, our proof of convergence is simpler as it simply requires to show that the problem is smooth and strongly-convex.
    \item All the theory from strongly-convex gradient descent apply to our problem. Thereby, we easily introduce acceleration -- being the first to do so -- outperforming the convergence rates for simple gradient descent. 
\end{itemize}

\section{Related Works}
\label{sec:related_work}

\begin{algorithm}[tbp]
    \caption{Distributed Randomized Power Iteration}
    \label{algo:distributed_power_iteration}
    \LinesNumberedHidden
    \DontPrintSemicolon
    \KwInput{Number of iteration $\alpha$ in $\N$.}
    \KwOutput{Global variable $\VV$ in $\R^{d \times r}$}
    \For{each client $i$ in $\OneToN$}{
    Generate a Gaussian matrix $\Phi^i$ in $\R^{n_i \times r}$. \;
    Compute $\VV^i = (\SS^i)^\top \Phi^i$. \;
    Share $\VV^i$. \;
    }
    Compute $\VV = \sum_\iN \VV^i$ using a secure aggregation protocol. \;
        Share $\VV$ with all clients.  \;
    \For{$a \in \{1, \dots, \alpha \}$} {
        \For{each client $i$ in $\OneToN$}{
            Compute $\VV^i = (\SS^i)^\top \SS^i \VV$. \;
            Share $\VV^i$. \;
        }
        Compute $\VV = \sum_\iN \VV^i$ using a secure aggregation protocol. \;
        Share $\VV$ with all clients.  \;
    }
\end{algorithm}

In this Section, we describe the distributed randomized power iteration and review related works on gradient descent.

\textbf{The power method} \citep{mises1929praktische} is a simple iterative algorithm used to approximate the dominant eigenvalue and the corresponding eigenvector of a square matrix by repeatedly multiplying the matrix by a vector and normalizing the result. To improve scalability to large problems, \citet{halko2011finding} have proposed \textit{Randomized SVD} a technique that sample a random Gaussian matrix $\Phi$ in $\R^{n \times r}$, multiply it by $(\SS^\top \SS)^\alpha \SS^\top$ ($\alpha$ in $\N$) to obtain a matrix $\VV$. While authors never mention the distributed setting, the adaptation is straightforward and we give the pseudo-code to compute $\VV$ in \Cref{algo:distributed_power_iteration}. Note that to avoid storing a $d \times d$ matrix, we compute the product $(\SS^\top \SS)^\alpha \SS^\top \Phi$ from right to left, resulting in the computation of a $d \times r$ matrix. The total number of computational operations is $(2\alpha + 1) ndr + (\alpha + 1) dr$.
Next, the authors either construct a QR factorization of $\VV$ to obtain an orthonormal matrix factorization of $\SS$, or compute its SVD decomposition to get the singular values.

\begin{remark}[Secure aggregation]
The distributed power iteration (\Cref{algo:distributed_power_iteration}) requires a federated setting with a central server that can perform a \emph{secure aggregation} \citep{bonawitz2017practical} of the local $(\VV^i)_{\OneToN}$. Extending to the decentralized setting is an interesting but out of scope direction.
\end{remark} 

\textbf{The idea of using gradient descent} to solve low-rank matrix factorization can be traced back to \citet{burer2003nonlinear,burer2005local}. Several theoretical results and convergence rate have been provided in the recent years \citep[][]{zhao2015nonconvex,tu2016low,zheng2016convergence,jain2017global,chen2019gradient,du2018algorithmic,ye2021global,jiang2023algorithmic}. \citet{ward2023convergence} has a rate depending on $\kappa^{-2}(\SS)$ using a power initialization with $\alpha = 0$ but only in the case of a low-rank matrix, for which it is known that we can have in fact a zero error (\Cref{prop:lower_bound_with_proba}). In contrast, our analysis in a strongly-convex setting allows, first, to plugging in any faster algorithms than simple gradient descent and thus to obtain a faster rate depending on $\kappa^{-1}$, and second, holds in the more general setting of a full-rank matrix. 

Note that \citet{ward2023convergence} do not mention the FL setting, however, their work naturally adapts itself to this setting as we have $\frac{1}{2} \sum_\iN \SqrdNrm{\SS^i - \UU^i \VV^\top}_\frob = \frac{1}{2} \| \SS - \UU \VV^\top \|_\frob^2$. On the other hand, it results in a high communication cost, given that it requires sharing at each iteration $k$ in $\N$ the matrix $\VV_k$ to compute the gradient (but does not require sharing $\UU_k$ which remains local), resulting in a communication cost of $O(K rd)$, where $K$ is the total number of iterations/communications.  Besides, $K$ depends on the condition number and thereby might be very large. On the contrary, our work and the one of \citet{halko2011finding} requires small (and potentially a single) communication steps.

Most of the articles extending gradient descent to the distributed setting \citep{hegedHus2016robust,zhu2019distributed,hegedHus2019decentralized,li2020global,li2021federated,wang2022federated} consider an approach minimizing the problem over a global variable $\VV$ and personalized ones $(\UU_i)_\iN$. This setup creates a global-local matrix factorization: $\VV$ contains information on features shared by all clients (item embedding), while $(\hat{\UU}^i)_\iN$ captures the unique characteristics of client $i$ in $\OneToN$ (user embeddings). We build on this approach in the next Section. 

\section{Practical Algorithm and Theoretical Analysis}
\label{sec:algo_gd}

In the next Subsections, we propose and analyze a parallel algorithm derived from \citet{halko2011finding} that combines power method and gradient descent.

\subsection{Combining Power Method with Parallel Gradient Descent}

In order to solve \eqref{eq:dist_MF}, a natural idea is to fix the matrix $\VV$ shared by all clients, and then compute locally the exact solution of the least-squares problem or run a gradient descent \emph{with respect to the variable $\UU$}. All computation are thus performed solely on clients once $\VV$ is initialized. This requires a number of communications equal to $\alpha + 1$ as the only communication steps occur when initializing $\VV$. In the low-noise or the low-rank regimes, taking $\alpha = 0$ allows to obtain $\epsilon = 0$ (\Cref{cor:bound_with_r_star}) and results to an \textit{embarrassingly parallel} algorithm. Besides, parallelizing the computation allows to go from a computational cost in $O(d r \sum_\iN n_i)$ to $O( d r n_i)$ which is potentially much smaller. Below, we give the optimal solution of \eqref{eq:dist_MF} for a fixed $\VV$. 

\begin{proposition}[Optimal solution for a fixed matrix $\VV$]
\label{prop:optimal_sol}
    Let $\VV \in \R^{d \times r}$ be a \emph{fixed} matrix, then \Cref{eq:dist_MF} is minimized for $ \hat{\UU}^i = \SS^i \VV (\VV^\top \VV)^{\dagger}$.
\end{proposition}

The main challenge in computing $(\hat{\UU}^i)_\iN$ is to (pseudo-)inverse $\VV^\top \VV$, which is known to be unstable under small numerical errors and potentially slow. 
We propose instead to do a local gradient descent on each client $i$ in $\OneToN$ in order to approximate the optimal $\hat{\UU}^i$ minimizing $F^i(\cdot, \VV)$. It results to a parallel algorithm that does not require any communication after the initialization of $\VV$. We give in \Cref{algo:gd_U} the pseudo-code: the server requires to do exactly $N (\alpha + 1) dr + dr^2$ computational operations and the client $i \in \OneToN$ carries out $(4T + 2\alpha + 2) n_i dr$ operations.

\begin{algorithm}[tbp]
    \caption{GD w.r.t. $\UU$ with a power init.}
    \label{algo:gd_U}
    \LinesNumberedHidden
    \DontPrintSemicolon
    \KwInput{Number of iteration $\alpha$ in $\N$, step-size $\gamma$.} 
    \KwOutput{$(U^i)_\iN$.}
    Run \Cref{algo:distributed_power_iteration} to compute $\VV = (\SS^\top \SS)^\alpha \SS^\top \Phi$. \;
    \For{each client $i$ in $\OneToN$ without any communication}{
        Sample a random matrix $\UU_0^i$ in $\R^{n_i \times r}$. \;
        \For{$t \in \{1, \dots, T$ \}}{
            Compute $\nabla_\UU F(\UU_{t-1}^i, \VV) = (\UU_{t-1}^i \VV^\top - \SS^i) \VV$. \\
            $\UU_t^i = \UU_{t-1}^i - \gamma \nabla_\UU F(\UU_{t-1}^i, \VV)  $. \;
        }
    }
\end{algorithm}

This approach offers much more numerical stability, allows for any kind of regularization, ensures strong guarantees of convergence, and explicit the exact number of epochs required to reach a given accuracy. A simple analysis of gradient descent in a smooth strongly-convex setting leads to a linear convergence toward a global minimum of the function $F^i(\cdot, \VV)$ with a convergence rate equal to $1-\kappa^{-2}(\VV)$, or even to $1-\kappa^{-1}(\VV)$ if a momentum is added.

\begin{remark}[$L_*$/$L_1$/$L_2$-regularization.] We can use regularization on $\UU$ with various norms: nuclear norm which yields a low-rank $\UU$, $L_1$-norm resulting in a sparse $\UU$, and $L_2$-norm leading to small values. The nuclear regularization requires to compute the SVD of $\UU$ to compute the gradient \citep{avron2012efficient}, which generates an additional $O(nr\log r)$ complexity.
\end{remark}

\subsection{Rate of Convergence of Algorithm 2}
\label{subsec:rate_cvg_algo}

The cornerstone of the analysis relies on using power initialization  (\Cref{algo:distributed_power_iteration}), which forces having $\VV$ in the column span of $\SS$. Besides, once $\VV$ is set by \Cref{algo:distributed_power_iteration}, for any client $i$ in $\OneToN$, $F^i(\cdot, \VV)$ is $L$-smooth and $\mu$-strongly-convex as proved in the following properties. 

\begin{property}[Smoothness]
\label{prop:smooth}
    Let $\VV$ in $\R^{d\times r}$ initialized by \Cref{algo:distributed_power_iteration}, then all $\bigpar{F^i(\cdot, \VV)}_\iN$ are $L$-smooth, i.e., for any $i$ in $\OneToN$, for any $\UU, \UU'$ in $\R^{n_i \times d}$, we have $\|\nabla F^i(\UU, \VV) - \nabla F^i(\UU', \VV)\|_\frob \leq L \|\UU - \UU'\|_\frob$, with $L = \sigma_{\max}^2(\VV)$.
\end{property}

\begin{proof}
\renewcommand{\qedsymbol}{}
Let $i$ in $\OneToN$ and $\UU, \UU'$ in $\R^{n_i\times r}$, we have that:
\begin{align*}
&\|\nabla F(\UU, \VV) - \nabla F(\UU', \VV) \|_\frob = \| (\UU - \UU') \VV^\top \VV \|_\frob \\
&\quad\overset{\text{Prop. S2}}{\leq} \sigma_{\max}(\VV^\top \VV) \|\UU - \UU'\|_\frob \\
&\qquad= \sigma_{\max}^2(\VV) \|\UU - \UU'\|_\frob\,.~\Box
\end{align*}
\end{proof}

\begin{property}[Strongly-convex]
\label{prop:stgly_cvx}
    Let $\VV$ in $\R^{d\times r}$ initialized by \Cref{algo:distributed_power_iteration}, then all $\bigpar{F^i(\cdot, \VV)}_\iN$ are $\mu$-strongly-convex, i.e., for any $i$ in $\OneToN$, for any $\UU, \UU'$ in $\R^{n_i \times d}$, we have $\pdtscl{\nabla F^i(\UU, \VV) - \nabla F^i(\UU', \VV)}{\UU - \UU'} \geq \mu \sqrdnrm{\UU - \UU'}_\frob$, with $\mu = \sigma_{\min}^2(\VV)$.
\end{property}

\begin{proof}
\renewcommand{\qedsymbol}{}
    Let $i$ in $\OneToN$ and $\UU, \UU'$ in $\R^{n_i\times r}$, we have that:
    \begin{align*}
        &\pdtscl{\nabla F^i(\UU, \VV) - \nabla F^i(\UU', \VV)}{\UU - \UU'} \\
        &\qquad= \Tr{(\UU - \UU')^\top (\UU - \UU') \VV^\top \VV} \\
        &\qquad= \sqrdnrm{(\UU - \UU') \VV^\top}_\frob \overset{\text{Prop. S2}}{\geq} \sigma_{\min}^2(\VV) \sqrdnrm{\UU - \UU'}_\frob \,.~\Box 
    \end{align*}
\end{proof}

This proves that the surrogate objective functions $(F^i(\cdot, \VV))_\iN$ are $\mu$-strongly-convex, with $\mu \geq 0$. As discussed in \Cref{thm:bound_on_kappa}, as long as $r < \rank{\SS}$, we have with high probability $\mu > 0$. This is always true in the noisy setting (full rank); otherwise, we simply need to decrease $r$.

\Cref{prop:smooth,prop:stgly_cvx} allows to use classical results in optimization and draws a linear rate of convergence depending on $\mu/L$, or $\sqrt{\mu / L}$ if we use acceleration. 

\begin{theorem}
\label{thm:cvg_rate}
    Under the distributed power initialization (\Cref{algo:distributed_power_iteration}), considering \Cref{prop:smooth,prop:stgly_cvx}, let $T$ in $\N^*$, $\gamma =1  / L$, then after running \Cref{algo:gd_U} for $T$ iterations, the excess loss function is upper bounded:
    $F^i(\UU_T^i) - F^i(\hat{\UU}^i) \leq   \bigpar{1 - \mu/L}^T (F^i(\UU_0^i) - F^i(\hat{\UU}^i))$, where $\mu / L = \kappa^{-2}(\VV)$.
    Using Nesterov momentum, this rate is accelerated to:
    $F^i(\UU_T^i) - F^i(\hat{\UU}^i) \leq  2 (1 - \sqrt{\mu/L})^T (F^i(\UU_0^i) - F^i(\hat{\UU}^i))$.
\end{theorem}

\begin{remark}[Convergence in a single iteration]
    If we orthogonalize $\VV$, then we can converge in one iteration as gradient descent reduces to Newton method. Our algorithm would be equivalent to the one proposed by \citet{halko2011finding}. However, orthogonalizing $\VV$ requires to compute the SVD or to run a gradient descent on~$\VV$.
\end{remark}

\Cref{thm:cvg_rate} establishes that $(F^i(\UU_t^i, \VV))_{t \in \N^*}$ converges to $F^i(\hat{\UU}^i, \VV)$ at a linear rate dominated by $\exp(-\kappa^{-1}(\VV))$ or $\exp(-\kappa^{-2}(\VV))$. A question then emerges, \emph{can we control $\kappa(\VV)$?} This parameter plays a pivotal role in determining the convergence rate and is affected by the sampled matrix $\Phi$. Consequently, an ill-conditioned matrix $\VV$ may arise, significantly hindering convergence. The below corollary gives a bound in probability on $\kappa^2(\VV)$ that depends on the spectrum of $\SS$ while being independent of the sampled $(\Phi^i)_\iN$.

\begin{theorem}
\label{thm:bound_on_kappa}
    Under the distributed power initialization (\Cref{algo:distributed_power_iteration}), considering \Cref{prop:smooth,prop:stgly_cvx}, for any $\mathrm{p}$ in $]0,1[$, with probability at least $1-3\mathrm{p}$, we have $\kappa(\VV)^2 < \kappa_{\mathrm{p}}^2$, with:
    \begin{align*}
        \kappa_{\mathrm{p}}^2 := \frac{1}{\mathrm{p}^2} \bigpar{ 9r^2 \frac{ \sigma^{2 (2 \alpha + 1)}_{\max}}{ \sigma^{2 (2 \alpha + 1)}_{r} } +  4 r \bigpar{d +\log(2\mathrm{p}^{-1})} \ffrac{ \sigma_{r+1}^{2 \alpha}}{\sigma_{r}^{2 \alpha}} } \,.
    \end{align*}
    With probability $\mathrm{P}$, if we sample $m = \lfloor - \log_2(1 - \mathrm{P}) \rfloor$ independent matrices $(\Phi_j)_{j=1}^m$ to form $\VV_j = \SS^\alpha \Phi_j$ and run \Cref{algo:gd_U}, at least one initialization results to a convergence rate upper bounded by $1 - \kappa_{1/6}^{-2}$. 
\end{theorem}

We can make the following remarks.
\begin{itemize}[leftmargin=*]
    \item \textbf{Impact of $\alpha$.} Increasing $\alpha$ enworse the rate of convergence as $\sigma^{2}_{\max} / \sigma^{2}_{r} \geq 1$. In the regime where $\sigma^{2}_{\max}$ is very large and $\sigma^{2}_{r} = \Omega(1)$, having $\alpha > 0$ drastically \emph{hinders} the gradient descent convergence. 
    However, in this case, it is possible to compute the exact solution of \eqref{eq:dist_MF}, further, as emphasized in \Cref{cor:bound_with_r_star} which showcases the interest of having $\alpha \neq 0$, increasing $\alpha$ allows reducing $\epsilon$, i.e., the gap between the approximation error (induced by taking $r \leq \rank{\SS}$) and the minimal reconstruction error $\epsilon_{\min}$. Therefore, there is a trade-off associated with the choice of $\alpha$; we illustrate it in the experiments on three real datasets: mnist, w8a and celeba-200k.
    \item \textbf{Asymptotic values of $\kappa(\VV)$.} In the regime $\sigma_{\max}, \sigma_r = O(1)$ and $\sigma_{r+1} \ll 1$, we have $1-\kappa_{\mathrm{p}}^{-2} = O(1 - p^2 / \mathrm{r}^2)$. Further, by employing acceleration techniques, we have an improved rate depending on $r^{-1}$ and~not~$r^{-2}$.
    \item \textbf{Ill-conditioned matrix}. Even if the matrix $\SS$ is ill-conditioned, the rate of convergence does not suffer from~$\sigma_{\min}$.
    \item \textbf{Result in probability.} The bound on $\kappa^2(\VV)$ is given with high probability $1 - 3 \mathrm{p}$, which flows from the concentration inequalities proposed by \citet{davidson2001local} and \citet{vershynin_2012} on the largest/smallest eigenvalues of a random Gaussian matrix.
    \item \textbf{Rotated matrix.} Note that the probability $\mathrm{p}$ is taken not on $\Phi$ but on the rotated matrix $\hat{\Phi} = \UU_*^\top \Phi$.
    \item \textbf{Almost sure convergence.} We propose to sample several matrix $\Phi$ until the condition number of $(\SS^\top \SS)^\alpha \SS^\top \Phi$ is good enough. By leveraging theory on random Gaussian matrix, we can compute the number of sampling $m$ to achieve this bound on $\kappa_{\mathrm{p}}$ with probability $\mathrm{p}$. Alternatively, we can repeatedly sample $\Phi$ until the condition number of $\VV$ falls below $\kappa_{\mathrm{p}}$. Since we know that this is achievable within a finite number of samples, we can obtain a convergence almost surely, with a rate dominated by $\exp(-\kappa_{\mathrm{p}}^{-2})$.
\end{itemize}

\Cref{thm:bound_on_kappa} states that $\kappa(\VV)$, which determines the convergence rate of \Cref{algo:gd_U}, can be upper-bounded with high probability. Next question is: \emph{how far is $F^i(\hat{\UU}, \VV)$ from the minimal possible error of reconstruction $\epsilon_{\min}$?} 
With a lower-bound established (\Cref{prop:lower_bound_with_proba}), the subsequent section endeavors to upper-bound the Frobenius-norm error when approximating $\SS$ with a rank-$r$ matrix. 

\subsection{Bound on the Frobenius-norm of the Error of Reconstruction} 
\label{subsec:optimal_solution_analysis}

In the case of a low-rank matrix $\SS$ with $\rank{\SS} = r$ (i.e., $\EE = 0$), the couple $(\hat{\UU}^i, \VV)_\iN$ allows to reconstruct $(\SS^i)_\iN$ without error. This can be proved using Theorem 9.1 from \citet{halko2011finding} and by underlining that for any client $i \in \OneToN$,  we have $\hat{\UU}^i \VV^\top = \SS^i \VV (\VV^\top \VV)^{-1} \VV^\top = \SS^i \PP$, where $\PP$ is a projector on the subspace spanned by the columns of~$\VV$.

\begin{proposition}
    \label{prop:low_rank_matrix}
    Let $r \in \{r_*, \dots, d \wedge n \}$, in the low-rank matrix regime where we have $\EE = 0$, using the power initialization (\Cref{algo:distributed_power_iteration}), we achieve $\min_{\UU^i \in \R^{n_i \times r} } \SqrdNrm{\SS^i - \UU^i \VV^\top}_\frob = 0$.
\end{proposition}

Second, we are interested in the full-rank matrix $\SS$ scenario and give below a theorem upper-bounding the Frobenius-norm error when approximating $\SS$ with a rank-r matrix. The proof requires (1) to show the link between the error of reconstruction and the diagonal elements of the projector on the subspace spanned by the columns of $\VV_*^\top\VV$, (2) upper bound it by the norm of a Gaussian vector and the smallest singular value of a Whishart matrix, and finally (3) apply concentration inequalities.

\begin{theorem}
\label{thm:upper_bound_with_proba}
    Let $r \leq d \wedge n$ in $\N^*$, using the power initialization (\Cref{algo:distributed_power_iteration}), for $\mathrm{p} \in ]0, 1[$, with probability at least $1 - 2\mathrm{p}$, we have:
    \begin{align*}
         &\min_{\UU \in \R^{n \times r} } \SqrdNrm{\SS - \UU \VV^\top}_\frob < \sum_{i >r} \sigma_i^2  \times \\
         &\qquad \bigpar{ 1 +  2 r \mathrm{p}^{-1}\bigpar{\ln(\mathrm{p}^{-2}) + \ln(2) r}\ffrac{(\sigma^2_{\max} - \sigma_i^2 ) }{\sigma_r^2} \frac{\sigma_i^{4 \alpha}}{\sigma_r^{4 \alpha}} }\,.
    \end{align*} 
\end{theorem}

The main takeaway is that our algorithm can run in a finite number of communication rounds.  We can make the following remarks.
\begin{itemize}[leftmargin=*]
    \item \textbf{Frobenius-norm.} This result \emph{on the Frobenius-norm} for $\alpha > 0$ is new and was missing in \citet{halko2011finding}. The rate controlling $\epsilon$ depends on $r^2$.
    \item \textbf{Comparison with \citet[][]{halko2011finding} in the case $\alpha = 0$.} In contrast, they obtain a dependency on $r$ for both the $2$-norm and the Frobenius norm when $\alpha = 0$. Another difference: our bound depends on the ratio $\sigma_{\max}^2 / \sigma_r^2$ unlike theirs (Theorem 10.7). But these two drawbacks are annihilated as soon as $\alpha>0$, we provide more details after \Cref{cor:bound_with_r_star}. 
    \item \textbf{Multiplicative noise.} Following \citet{halko2011finding}, we have a multiplicative noise. w.r.t. to the minimal possible error $\epsilon_{\min}$.
    \item \textbf{Range of $r$.} Contrary to \citet{halko2011finding,hardt2014noisy,li2021communication}, this theorem holds for any value $r$. However, for $r < r_*$ (corresponds to not taking the whole signal into account), the bound is very large as $\sigma_i = \Omega(1)$, this is why we give a corollary with $r \geq r_*$ in \Cref{cor:bound_with_r_star}.    
\end{itemize}

The upper bound given in \Cref{thm:upper_bound_with_proba} is minimized if $\sigma_{\max}^2 = \sigma_r^2+ o(1)$ and if for $i > r$ we have $\sigma_i^2 \ll \sigma_r^2$, which we assume to be the case for $r = r_*$. Therefore, taking $r = r_*$ in \Cref{thm:upper_bound_with_proba} minimizes the provided bound.
However, the error of reconstruction can only be reduced if taking more than $r_*$ components. Indeed, it mathematically corresponds to having two projectors $\PP, \PP'$ on the subspaces spanned by the $r_*$ or $r$ components; therefore we have $\mathrm{Im}(\PP') \subset \mathrm{Im}(\PP)$. In particular, it means that if $r > r_*$, the error $\epsilon$ will be lower than in the case $r = r_*$. This results in a tighter bound for any $r_* \leq r \leq n \wedge d$. 
Additionally, we consider $p=1/4$ in \Cref{thm:upper_bound_with_proba} in order to obtain a bound on the Frobenius-norm with probability at least $1/2$ and derive a number of sampling $m$ s.t. the bound is verified for at least one sampled matrix $\Phi$ with probability $\mathrm{P} \in ]0, 1[$.

\begin{corollary}
\label{cor:bound_with_r_star}
    For any $ r_* \leq r \leq n \wedge d$, for $\alpha \in \N^*$, with probability $\mathrm{P} \in ]0, 1[$, if we sample $m = \lfloor - \log_{2} (1-\mathrm{P}) \rfloor$ independent matrices $(\Phi_{j})_{j=1}^m$ to form $\VV_j = \SS^\alpha \Phi$ and $\UU_j = \SS \VV_j (\VV_j \VV_j)^{-1} \VV_j$, at least one of the couple $(\UU_j, \VV_j)$ results in verifying $\sqrdnrm{\SS - \UU_j \VV_j}_\frob < \epsilon + \sum_{i >r_*} \sigma_i^2$, with: 
    \begin{align}
    \label{eq:upper_bound_with_proba_1_on_2}
        &\epsilon = \sum_{i >r_*} \sigma_i^2 \bigpar{32 \ln(4) r_* (r_*+1) \ffrac{(\sigma^2_{\max} - \sigma_i^2 ) }{\sigma_{r_*}^2} \frac{\sigma_i^{4 \alpha}}{\sigma_{r_*}^{4 \alpha}}} \,.
    \end{align} 
\end{corollary}

We can make the following remarks.
\begin{itemize}[leftmargin=*]
    \item \textbf{Dominant term.} The dominant term for $\epsilon$ is proportional to $r_*^2$ while it is proportional to $r_*$ for \citet{halko2011finding} in the scenario $\alpha = 0$. Nonetheless, the exacerbated $r_*^2$ rate is mitigated by its appearance within a logarithmic as we have  $ \alpha = \Omega\bigpar{\frac{\log( \sigma_{\max} d r_*^2  \epsilon^{-1})}{\log(\sigma_{r_*}) - \log(\sigma_{r_*+1}) }}$. In other words, doubling $\alpha$ yields an equivalent rate. 
    \item \textbf{Impact of  $\alpha$.} Given that for any $i$ in $\{r_{*} +1, \dots, n \wedge d\}$, we have $\sigma_i \leq \sigma_{r_*}$, increasing $\alpha$ reduces $\epsilon$ by a factor $\sigma_{r_* +1}^{4\alpha}/\sigma_{r_*}^{4\alpha}$ and improves the convergence of the algorithm towards the minimal error $\epsilon_{\min}$. This has a major impact in the particular regime underlined by \Cref{cor:particular_case}, \emph{where we get a finite number of communication rounds!}
    \item \textbf{``Comparison'' with related works.} \citet{halko2011finding} provide a result solely on the $2$-norm which is not comparable to our Frobenius-norm: they obtain $\alpha = \Omega(\frac{\sigma_{r_* + 1} \log(d)}{\epsilon})$.
    Our asymptotic rate on $\alpha$ would be better if the quantities were comparable.  \citet{li2021communication} provide a result solely on the $2$-norm distance between eigenspaces which is again not comparable: they obtain $\alpha = \Omega\bigpar{\frac{\sigma_{r_*} \log(d \epsilon^{-1})}{\sigma_{r_*} - \sigma_{r_*+1}}}$. In the regime where $\sigma_{r_* + 1} \ll 1$, it can not be equal to zero, unlike us. Note that however in the regime $\sigma_{r_*} \approx \sigma_{r_+ 1 }$, the two bounds would be equivalent.
    \item \textbf{Value of $m$.} Sampling $m = 10$ random matrices is enough to have at least one of them resulting to verify \Cref{eq:upper_bound_with_proba_1_on_2} with probability $1 - 10^{-3}$.  
\end{itemize}

The next corollary emphasizes the special regime of a full-rank matrix $\SS$ s.t. $\sigma_{\max}, \dots, \sigma_{r_*} = o(1)$ and $\sigma_{r_* + 1} \ll 1$, which is of great interest as raised by \Cref{cor:bound_with_r_star}. In this regime, increasing $\alpha$ and using gradient descent is particularly efficient in terms of communication cost~and~precision~$\epsilon$. 

\begin{corollary}[Full-rank scenario with small $\sigma_{\max} / \sigma_{r_*}$]
\label{cor:particular_case}
    Let $\lambda, \xi > 0$. Suppose the first $r_*$ (resp. the last $n \wedge d - r_*$) singular values of $\SS$ are equal to a large $\lambda$ (resp. to a small $\xi$), then \Cref{algo:gd_U} has a linear rate of convergence determined by $\kappa^2 \leq O(r_*^2 + rd \xi / \lambda) = O(r_*^2)$, and we have an error $\epsilon = O(dr_*^2\xi^{4\alpha + 2} / \lambda^{4\alpha})$.
\end{corollary}

In the next section, we illustrate the insights highlighted by our theorems. on both synthetic and real data. 
In particular, on \Cref{fig:with_noise,fig:synth}, we illustrate the setting of \Cref{cor:particular_case} with $\lambda = 1$, $\xi = 10^{-6}$, $d=200$, and $r_* = 5$ resulting to $\kappa^2 \approx 25$ and $\epsilon_{\alpha = 0} \approx 5^{-7}$ for $\alpha = 0$; in contrast, for $\alpha=1$, we have $\epsilon_{\alpha = 1} \approx 5^{-37}$ after only two communications!

\section{Experiments}
\label{sec:experiments}

Our code is provided on Github. Experiments have been run on a 13th Gen Intel Core i7 processor with 14 cores. Our benchmark is SVD in the centralized setting (SVD of concatenated matrices)  using \texttt{scipy.linalg.svd}, this corresponds to the green line in all experiments.

\begin{table}
\centering
\begin{tabular}{lccc}
Settings & mnist & celeba & w8a \\
\hline 
dimension $d$ & $784$ & $96$ & $300$\\
latent dimension $r$ & $20$ & $20$ & $20$ \\
training dataset size & $6000$  & $557$ & $49,749$\\
number of clients $N$ & $10$ & $25$ & $25$ \\
\bottomrule
\end{tabular}
\caption{Settings of the experiments.}
\label{tab:settings}
\end{table}

\textbf{Synthetic dataset generation.} We consider synthetic datasets with $N = 25$ clients. For each client $i$ in $\OneToN$, we have $n_i = 200$ and $d=200$. We build a global matrix  $\SS = \XX + \EE$ and then split it across clients. We set $\rank{\XX} = 5$ with $\XX = \UU_\XX \mathbf{\SSigma}_\XX \VV_\XX$, where $\UU_\XX$ (resp. $ \VV_\XX$) are in $\mathcal{O}_{n}(\R)$ (resp. $\mathcal{O}_d(\R)$). $\mathbf{\SSigma}_\XX$ is a diagonal matrix in $\R^{n \times d}$ with the $5$ first values equal to $1$, and the other to $0$. $\EE$ is the noise matrix which elements are independently sampled from a zero-centered Gaussian distribution. 

\begin{figure}
    \centering
    \includegraphics[width=0.37\textwidth]{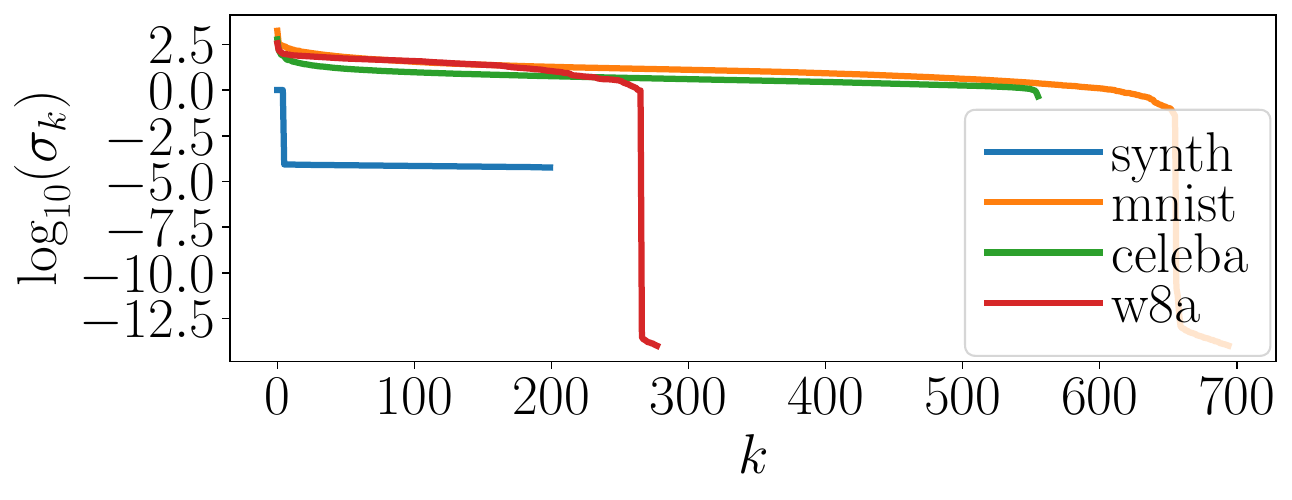}
    \caption{Singular values of the four used datasets.}
    \label{fig:svd}
\end{figure}

\textbf{Real datasets.} We consider three real datasets: mnist \citep{lecun2010mnist}, celeba-200k \citep{liu2015deep} and w8a \citep{chang2011libsvm}. We do not use the whole datasets for mnist and celeba-200k. For mnist (resp. celeba-200k), clients receive images from a single digit (resp from a single celebrity). For w8a, the dataset is split randomly across clients. This results in two image datasets with low (mnist) or high (celeba-200k) complexity with heterogeneous clients, and a tabular dataset with~homogeneous~ones.

We plot the SVD decomposition of each dataset on \Cref{fig:svd}. 
On \Cref{fig:wrt_cond_number}, we run experiments on the synthetic dataset with $r = \rank{\XX} = 5$ for $50$ different samples of $\Phi$. We plot $\sigma^2_{r_*}(\VV) / \sigma^2_{\max}(\VV)$ on the X-axis, and the logarithm of the loss $F$ after $1000$ local iterations on the Y-axis. The goal is therefore to illustrate the impact of the sampled $\Phi$ on the convergence rate.  

On \Cref{fig:wrt_number_of_its}, we run experiments on the four different datasets; we plot the iteration index on the X-axis, and the logarithm of the loss $F$ w.r.t. iterations on the Y-axis. We run a single gradient descent after having sampled $m=20$ random matrices $\Phi$ to take the one resulting in the best condition number $\kappa(\VV)$. We run experiments w./w.o. a momentum $\beta_k = k/(k+3)$, with $k$ the iteration index. The goal is here to illustrate on real-life datasets how the algorithm behaves in~practice. In Table~\ref{tab:exp_related_work}, we contrast our approach with existing algorithms using gradient descent \citep{ward2023convergence,ye2021global} by giving the number of communication to reach a given error $\epsilon + \epsilon_{\min}$, showing its superiority in terms of communication.

\textbf{Observations.}
\begin{itemize}[leftmargin=*]
    \item \Cref{fig:wrt_cond_number} shows the validity of \Cref{thm:cvg_rate,thm:bound_on_kappa} in the scenario without momentum as we obtain a linear convergence rate determined by $\sigma^2_{r_*}(\VV) / \sigma^2_{\max}(\VV)$.
    \item In the low-rank regime (\Cref{fig:without_noise}), we recover $\epsilon = 0$ as stated by \Cref{prop:low_rank_matrix} (up to machine precision). 
    \item In the regime of a full-rank scenario with a small $\sigma_{\max} / \sigma_{r_*}$ (\Cref{fig:synth,fig:with_noise}, scenario highlighted by \Cref{cor:particular_case}), gradient descent is fast (\Cref{thm:bound_on_kappa}) and recovers the exact solutions $(\hat{\UU}^i, \VV)_\iN$ of \eqref{eq:dist_MF}. Further, in this regime having $\alpha \geq 1$ reduces $\epsilon$ (\Cref{thm:upper_bound_with_proba,cor:bound_with_r_star}) which is observed in practice. For such a setting, gradient descent with $\alpha \geq 1$ is the best strategy.
    \item In the regime of a full-rank scenario with a large $\sigma_{\max} / \sigma_r$ (\Cref{fig:celeba,fig:mnist,fig:w8a}), using $\alpha \geq 1$ might lead to a very slow convergence rate. For illustration, on w8a (resp. mnist and celeba-200k) $\kappa(\VV)$ is equal to $4$ (resp to $10$ and $20$) for $\alpha = 0$, and equal to $120$ (resp. to $5000$ and $20000$) for $\alpha = 1$. In practice, the slow convergence rate is observed. However, as taking $\alpha \geq 1$ reduces $\epsilon$ ( \Cref{thm:upper_bound_with_proba,cor:bound_with_r_star}, and observed on \Cref{fig:celeba,fig:mnist,fig:w8a}), it might be preferable in this regime to compute the exact solution of \eqref{eq:dist_MF} with a pseudo-inverse rather than running a gradient descent. If it is not possible to compute the pseudo-inverse or if regularization is used, mandating the use of gradient descent, then it is preferable to take~$\alpha = 0$.
\end{itemize}

\begin{figure}
    \centering     
    \begin{subfigure}{0.495\linewidth}
        \includegraphics[width=1\linewidth]{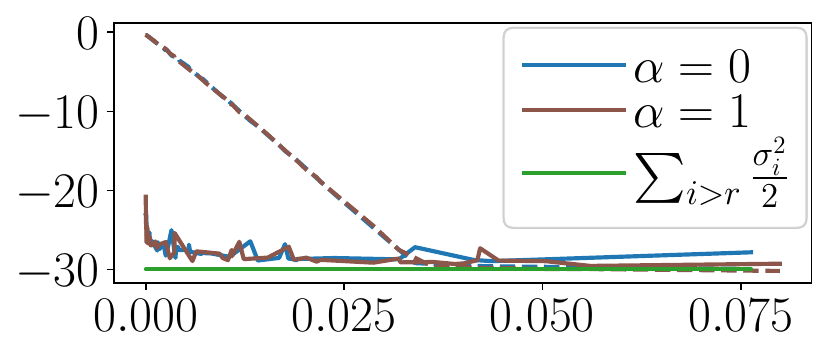}
        \caption{$\EE^i_{kl} = 0$}
        \label{fig:without_noise}
    \end{subfigure}
    \begin{subfigure}{0.495\linewidth}
       \includegraphics[width=1\linewidth]{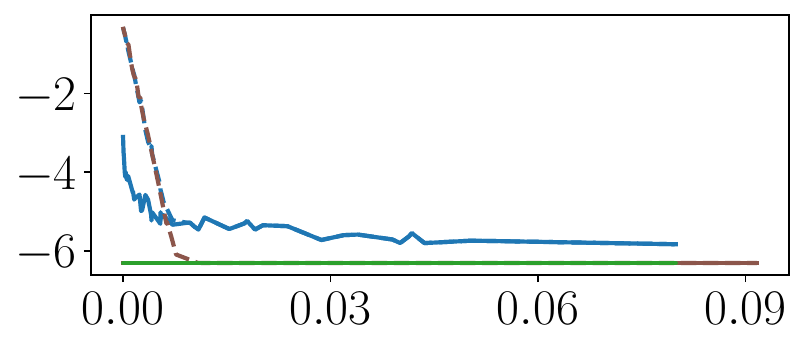}
        \caption{$\EE^i_{kl} \sim \mathcal{N}\big(0, (10^{-6})^2 \big)$}
        \label{fig:with_noise}
    \end{subfigure}
    \caption{ Matrix factorization of a low-rank matrix (left) and full-rank (right). We sample $m = 50$ different $\Phi$. X-axis: $\kappa^{-2}(\VV)$. Y-axis: logarithm error $\log_{10}(\sqrdnrm{\SS - \UU \VV^\top}_\frob)$. Plain line: exact solution. Dashed line: gradient descent after $K=1000$ iterations for each sampled $\Phi$.}
    \label{fig:wrt_cond_number}
\end{figure}

\begin{table}
\centering
\begin{tabular}{lcccc}
Algorithms & synth & w8a & mnist & celeba \\
\hline 
$\alpha=0$ & $1$ & $1$ & $1$ & $1$\\
WK2023 & $26$ & $6$& $35$& $58$ \\
YD2021 & \tiny{$\geq10^{13}$} & \tiny{$\geq10^{20}$} & \tiny{$\geq10^{20}$} & \tiny{$\geq10^{21}$} \\
\bottomrule
Reached error & $-5.5$ & $5.5$& $4.5$ & $5$ \\ 
\end{tabular}
\caption{Number of communications to reach error $\epsilon + \epsilon_{\min}$.}
\label{tab:exp_related_work}
\end{table}

\begin{figure}
    \centering     
    \begin{subfigure}{0.24\linewidth}
        \includegraphics[width=1\linewidth]{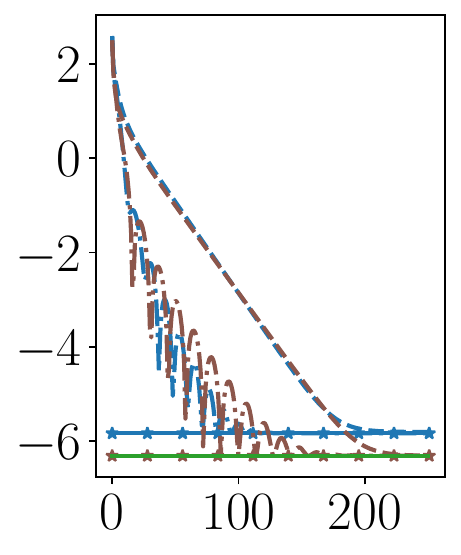}
    \caption{synthetic}
    \label{fig:synth}
    \end{subfigure}
    \begin{subfigure}{0.24\linewidth}
       \includegraphics[width=1\linewidth]{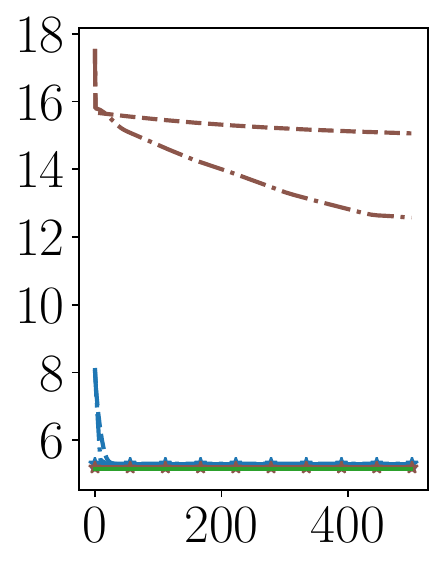}
        \caption{w8a}
        \label{fig:w8a}
    \end{subfigure}
    \begin{subfigure}{0.24\linewidth}
       \includegraphics[width=1\linewidth]{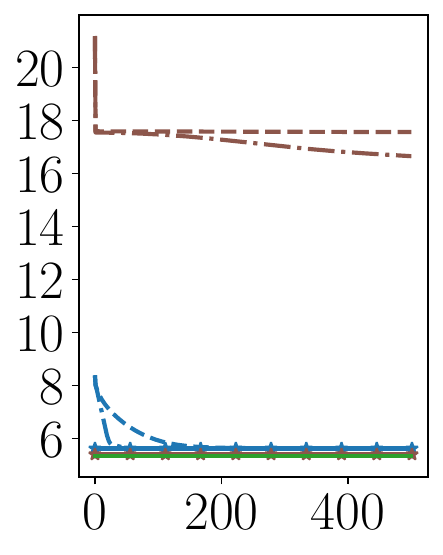}
        \caption{mnist}
        \label{fig:mnist}
    \end{subfigure}
    \begin{subfigure}{0.24\linewidth}
       \includegraphics[width=1\linewidth]{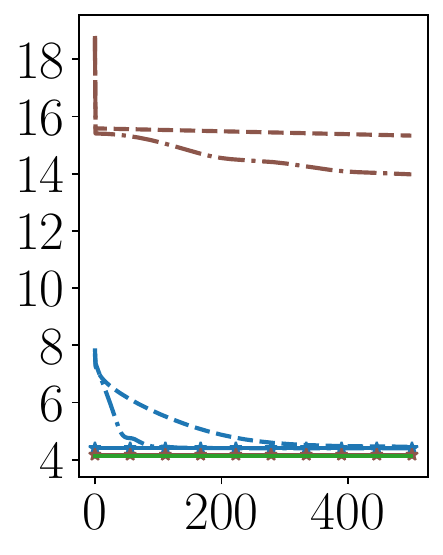}
        \caption{celeba}
        \label{fig:celeba}
    \end{subfigure}
    \caption{Convergence plot on four datasets. X-axis: number of iterations. Y-axis: logarithm error $\log_{10}(\sqrdnrm{\SS - \UU \VV^\top}_\frob)$. Plain line: exact solution. Dashed line: gradient descent. Dashed-dotted line: gradient descent with momentum $\beta_k = k/(k+3)$, with $k$ the iteration index.}
    \label{fig:wrt_number_of_its}
\end{figure}

\section{Conclusion}
\label{sec:conclusion}

In this article, we propose an in-depth analysis of low-rank matrix factorisation algorithm within a federated setting. We propose a variant of the power-method that combines a global power-initialization and a local gradient descent, thus, resulting to a smooth and strongly-convex problem. This setup allows for a finite number of communications, potentially even just a single one. We emphasize and experimentally illustrate the regime of high interest raised by our theory (\Cref{cor:particular_case}). Finally, drawing from \Cref{thm:bound_on_kappa,thm:cvg_rate,thm:upper_bound_with_proba,cor:bound_with_r_star}, we highlight the following key insights~from~our~analysis.

\begin{takeaway}
    Increasing the number of communication $\alpha$ leads to reduce the error $\epsilon$ by a factor $\sigma_{r_*+1}^{4\alpha}/ \sigma_{r_*}^{4\alpha}$, therefore, getting closer to the minimal Frobenius-norm error $\epsilon$.
\end{takeaway}

\begin{takeaway}
    Using a gradient descent instead of an SVD to approximate the exact solution of the strongly-convex problem allows us to bridge two parallel lines of research. Further, we obtain a simple and elegant proof of convergence and all the theory from optimization can be plugged in.
\end{takeaway}

\begin{takeaway}
    By sampling several Gaussian matrices $\Phi$, we improve the convergence rate of the gradient descent. Further, based on random Gaussian matrix theory, it results in a almost surely convergence if we sample $\Phi$ until $\VV$ is well conditioned.
\end{takeaway}

\begin{drawback}
    If gradient descent (\Cref{algo:gd_U}) is used to compute $(\hat{\UU}^i)_{i \OneToN}$, and if we are in the regime where $\sigma_{\max} / \sigma_r \gg 1$, then increasing $\alpha$ results in increasing the condition number and therefore the number of local iterations. Furthermore, in the scenario $\alpha = 0$, the upper bound on $\epsilon$ given by \citet{halko2011finding} is better than ours.
\end{drawback}

Three open directions to this work can be considered. First, it would be interesting to consider the case of decentralized clients where it is not possible to compute a global V across the whole network. Second, instead of computing an exact $\VV$, one could compute an approximation of $\VV$ to reduce further the communication cost, this would lead to stochastic-like gradient descent. Third, the extension of our approach to matrix completion is non-trivial (as we require $\VV$ to be in the span of $\SS$) and could have a lot of interesting applications.

\section*{Acknowledgments}
This work was supported by the French government managed by the Agence Nationale de la Recherche (ANR) through France 2030 program with the reference ANR-23-PEIA-005 (REDEEM project). It was also funded in part by the Groupe La Poste, sponsor of the Inria Foundation, in the framework of the FedMalin Inria Challenge. Laurent Massoulié was supported by the French government under management of Agence Nationale de la Recherche as part of the “Investissements d’avenir” program, reference ANR19- P3IA-0001 (PRAIRIE 3IA Institute).

\bibliography{main.bib}

\onecolumn
\newpage 
\twocolumn

 \begin{center}
	     {\Large{\bf Supplementary material}}
 \end{center}

\appendix

In this appendix, we provide additional information to supplement our work. In \Cref{app:sec:classical_ineq}, we reminds some classical results on matrices. In \Cref{app:sec:proof_gd}, we give the demonstration of \Cref{thm:bound_on_kappa}, in \Cref{app:sec:proof_exact_sol}, we provide the proofs of \Cref{thm:upper_bound_with_proba,cor:bound_with_r_star}, and in \Cref{app:sec:deeper_related_work}, we provide a deeper description of the related works.
	
\setcounter{equation}{0}
\setcounter{figure}{0}
\setcounter{table}{0}
\setcounter{theorem}{0}
\setcounter{lemma}{0}
\setcounter{remark}{0}
\setcounter{proposition}{0}
\setcounter{property}{0}
\setcounter{definition}{0}

\renewcommand{\theequation}{S\arabic{equation}}
\renewcommand{\thefigure}{S\arabic{figure}}
\renewcommand{\thetheorem}{S\arabic{theorem}}
\renewcommand{\thelemma}{S\arabic{lemma}}
\renewcommand{\theproposition}{S\arabic{proposition}}
\renewcommand{\thecorollary}{S\arabic{corollary}}
\renewcommand{\thedefinition}{S\arabic{definition}}
\renewcommand{\theproperty}{S\arabic{property}}
\renewcommand{\theremark}{S\arabic{remark}}
\renewcommand{\thetable}{S\arabic{table}}


\section{Classical inequalities}
\label{app:sec:classical_ineq}

In this section, we recall some well-known results and inequalities.

\subsection{Results on matrices}

The first inequalities give a lower/upper bound of the norm of a matrix-vector product based on the lowest/largest singular value of the matrix.

\begin{proposition}
\label{prop:bound_matrix_vector_norm_2}
    Let $\AA \in \R^{n\times r}$ and $x \in \R^r$, then $\sigma_{\min}(\AA) \|x\|_2 \leq \|\AA x\|_2 \leq \sigma_{\max}(\AA) \|x\|_2$.
\end{proposition}

\begin{proof}
    By definition, $\sigma_{\min}(\AA) = \inf_{x' \in \R^r} \ffrac{\| \AA x'\|_2}{\|x'\|_2} \leq \ffrac{\| \AA x\|_2}{\|x\|_2}\leq \sup_{x' \in \R^r} \ffrac{\| \AA x'\|_2}{\|x'\|_2} = \sigma_{\max}(\AA)$.
\end{proof}

The second classical proposition states a lower/upper bound of the Frobenius-norm of matrix-matrix product based on their lowest/largest singular values. 

\begin{proposition}
\label{prop:bound_normfrob_product_matrix}
    Let $\AA \in \R^{n\times r}$ and $\BB \in \R^{r \times d}$, then:
    \begin{align*}
        &\max\bigpar{\sigma_{\min}(\BB) \|\AA\|_\frob, \sigma_{\min}(\AA) \|\BB\|_\frob} \leq \|\AA \BB \|_\frob\,,
    \end{align*}
and $\|\AA \BB \|_\frob \leq \min\bigpar{\sigma_{\max}(\BB) \|\AA\|_\frob, \sigma_{\max}(\AA) \|\BB\|_\frob}$.
\end{proposition}

\begin{proof}
Recall that for any matrix $\mathbf{X} = (\mathbf{X}_1, \dots, \mathbf{X}_d) \in \R^{n \times d}$, we have 
$\sqrdnrm{\mathbf{X}}_\frob = \Tr{\mathbf{X}^\top \mathbf{X}} = \sum_{i=1}^d \|\mathbf{X}_i \|_2^2\,.$
Then, noting $(\AA \BB)_i$ (resp $\BB_i$) the $i$-th column of $\AA\BB$ (resp. $\BB$), we write:
\begin{align*}
    \sqrdnrm{\AA \BB}_F &= \sum_{i=1}^d \|(\AA \BB)_i \|_2^2 = \sum_{i=1}^d \|\AA \BB_i \|_2^2 \\
    &\stackrel{\mathrm{prop.~ \ref{prop:bound_matrix_vector_norm_2}}}{\geq} \sum_{i=1}^d \sigma_{\min}(\AA) \|\BB_i\|_2^2 = \sigma_{\min}(\AA) \sqrdnrm{\BB}_F \,.
\end{align*}
And by the cyclic invariance of the trace, we also have $\sqrdnrm{\AA \BB}_F \geq \sigma_{\min}(\BB) \sqrdnrm{\AA}_F$. Identically, we obtain the upper bound involving the biggest singular value.
\end{proof}

The last proposition of this Subsection gives an equivalent characterization of two projectors using their image or their norm.

\begin{proposition}
\label{app:prop:proj_ineq}
    Let $\PP, \PP'$ in $\R^{d \times d}$ two projectors, then the following proposition are equivalent:
    \begin{enumerate}
        \item $\mathrm{Im}(\PP) \subset \mathrm{Im}(\PP')$,
        \item $\forall x \in \R^{d \times d},~\sqrdnrm{\PP x} \leq \sqrdnrm{\PP'x}$.
    \end{enumerate}
\end{proposition}

\subsection{Concentration inequalities}

The below Chernoff inequality \citep{chernoff1952measure} is an exponentially decreasing upper bound on the tail of a random variable based on its moment generating function.

\begin{proposition}[Chernoff inequality]
\label{app:tech:prop:chernoff}
Let $X$ be a random variable, from Markov's inequality, for every $t>0$, we have $\Proba{X \geq \delta} \leq \ffrac{\fullexpec{e^{t X}}}{e^{t \delta}}$, and for every $t<0$, we similarly have $
\Proba{X \leq \delta} \leq \ffrac{\fullexpec{e^{t X}}}{e^{t \delta}}$.
\end{proposition}

Next inequality is taken from \citet{vershynin_2012} and characterizes the extreme singular values of random Gaussian matrices.

\begin{proposition}[Wishart distributions, see Theorem 5.32 and Proposition 5.35
  from \citet{vershynin_2012}] 
\label{app:prop:whishart_concentration_ineq}
  Let~$r \leq d$ and $\XX$ be a $d \times r$ matrix whose entries follow independent standard normal distributions. For any number~$t \geq 0$, we have:
$$
\Proba{\sigma_{\min}(\XX) \geq \sqrt{d}-\sqrt{r}-t } \geq 1 - \exp(-t^2 / 2) \,,
$$
and
$$
\Proba{\sigma_{\max}(\XX) \leq \sqrt{d} + \sqrt{r} + t } \geq 1 - \exp(-t^2 / 2) \,.
$$
Therefore, for every $t \geq 0$, with probability at least $1 - 2 \exp(-t^2 / 2)$, one has:
\begin{align*}
    \sqrt{d}-\sqrt{r}-t \leq \sigma_{\min}(\XX) \leq \sigma_{\max}(\XX) \leq \sqrt{d} + \sqrt{r} + t \,.
\end{align*}

\end{proposition}

The final concentration inequality is taken from \citet{chen2005condition} and gives an upper bound for the tails of the condition number distributions of random rectangular Gaussian matrices.

\begin{proposition}[Tail of Wishart distributions, see Lemma 4.1
  from \citet{chen2005condition}] 
\label{app:prop:whishart_concentration_tail_ineq}
  Let~$r \leq d$ and $\XX$ be a $d \times r$ matrix whose entries follow independent standard normal distributions. For any number~$t > 0$, we have:
\begin{align*}
&\Proba{\frac{\sigma_{\max}(\XX)}{\sigma_{\min}(\XX)} \geq t, \sigma_{\min}(\XX) \leq  \frac{\sqrt{d}}{t} }  < \Proba{\sigma_{\min}(\XX) <  \frac{\sqrt{d}}{t} } \\
&\qquad< \frac{1}{\Gamma(d-r+2)} \bigpar{\frac{d}{t}}^{d- r + 1} \,,
\end{align*}

where $\Gamma$ is the Gamma function. In particular, if $\XX$ is a $r \times r$ matrix, we obtain:
\begin{align*}
\Proba{\frac{\sigma_{\max}(\XX)}{\sigma_{\min}(\XX)} \geq t, \sigma_{\min}(\XX) \leq  \frac{\sqrt{r}}{t} }  &< \Proba{\sigma_{\min}(\XX) <  \frac{\sqrt{r}}{t} } \\
&< \frac{r}{t} \,,
\end{align*}
and furthermore, we have $\Proba{\kappa_{\max}(\hat{\Phi}_{\leq r}) >  r^2 t^2} < 3 / t$.

\end{proposition}


\section{Proof of Subsection 3.2}
\label{app:sec:proof_gd}

In this Subsection, we give the demonstration of \Cref{thm:bound_on_kappa}.

\begin{theorem}
    Under the distributed power initialization (\Cref{algo:distributed_power_iteration}), considering \Cref{prop:smooth,prop:stgly_cvx}, for any $\mathrm{p}$ in $[0,1]$, with probability at least $1-3\mathrm{p}$, we have $\kappa(\VV)^2 < \kappa_{\mathrm{p}}^2$, with:
    \begin{align*}
        \kappa_{\mathrm{p}}^2 := \frac{1}{\mathrm{p}^2} \bigpar{ 9r^2 \frac{ \sigma^{2 (2 \alpha + 1)}_{\max}}{ \sigma^{2 (2 \alpha + 1)}_{r} } +  4 r \bigpar{d +\log(2\mathrm{p}^{-1})} \ffrac{ \sigma_{r+1}^{2 \alpha}}{\sigma_{r}^{2 \alpha}} } \,.
    \end{align*}

    Furthermore, with probability $\mathrm{P}$, if we sample $m = \lfloor - \log_2(1 - \mathrm{P}) \rfloor$ independent matrices $(\Phi_j)_{j=1}^m$ to form $\VV_j = \SS^\alpha \Phi_j$ and run \Cref{algo:gd_U}, at least one initialization results to a convergence rate upper bounded by $1 - \kappa_{\mathrm{p}}^{-2}$. 
\end{theorem}

\begin{proof}
    Under the distributed power initialization (\Cref{algo:distributed_power_iteration}), we have $\VV =  (\SS^\top \SS)^\alpha \SS^\top \Phi$, it follows that $\kappa(\VV) \leq \kappa((\SS^\top \SS)^\alpha \SS^\top) \kappa(\Phi)$. One could next use the concentration inequality from \citet{vershynin_2012} recalled in \Cref{app:prop:whishart_concentration_ineq} to obtain with probability upper than $1 - 2 \e^{-t^2 / 2}$ ($0 \leq t \leq \sqrt{d}-\sqrt{r}$):
    \begin{align*}
        \kappa(\Phi) \leq \bigpar{\frac{\sigma_{\max}}{\sigma_{\min}}}^{2\alpha +1} \frac{\sqrt{d}(1 + \sqrt{r / d} + t / \sqrt{d})}{\sqrt{d}(1-\sqrt{r / d}-  t / \sqrt{d})}\,.
    \end{align*}

    However, in the case of an ill-conditioned matrix, we might prefer to not depend on $\kappa^{-1}(\SS)$. Thus, we write instead:
    \begin{align*}
        \kappa^{2}(\VV) &\overset{\mathrm{(i)}}{=} \kappa^{2}(\VV_* \tilde{\Phi}) \overset{\mathrm{(ii)}}{=} \kappa^{2}(\tilde{\Phi}) = \ffrac{\max_{x \in \R^r, \sqrdnrm{x} = 1} \sqrdnrm{\tilde{\Phi} x}}{\min_{x' \in \R^r, \sqrdnrm{x'} = 1} \sqrdnrm{\tilde{\Phi} x'}} \\
        &\overset{\mathrm{(iii)}}{=} \ffrac{\max_{x \in \R^r, \sqrdnrm{x} = 1} \sum_{i=1}^d \sigma_i^{2 \tilde{\alpha}} (\hat{\Phi}_i^\top x)^2}{\min_{x' \in \R^r, \sqrdnrm{x'} = 1} \sum_{j=1}^d \sigma_j^{2 \tilde{\alpha}} (\hat{\Phi}_j^\top x')^2} \\
        &\leq \ffrac{\max_{x \in \R^r, \sqrdnrm{x} = 1} \sigma_{\max}^{2 \tilde{\alpha}} \sum_{i=1}^r (\hat{\Phi}_i^\top x)^2}{\min_{x' \in \R^r, \sqrdnrm{x'} = 1} \sigma_r^{2 \tilde{\alpha}} \sum_{j \leq r}  (\hat{\Phi}_j^\top x')^2}\\
        &\qquad + \ffrac{\max_{x'' \in \R^r, \sqrdnrm{x''} = 1} \sigma_{r+1}^{2 \tilde{\alpha}} \sum_{i=r+1}^d (\hat{\Phi}_i^\top x'')^2}{\min_{x' \in \R^r, \sqrdnrm{x'} = 1} \sigma_r^{2 \tilde{\alpha}} \sum_{j \leq r}  (\hat{\Phi}_j^\top x')^2}
    \end{align*}
    
    where at (i) we write $\VV = \VV_* \tilde{\Phi}$ which results to $\tilde{\Phi} := \SSigma^{2\alpha+1} \UU_*^\top 
    \Phi$, at (ii) we consider that $\VV_*$ is in $\mathcal{O}_d$ and at (iii) we define $\tilde{\alpha} = 2 \alpha + 1$ and $\hat{\Phi} := \UU_*^\top \Phi = \begin{pmatrix} \hat{\Phi}_1^\top  \\ \vdots  \\ \hat{\Phi}_d^\top  \end{pmatrix}$ has the same distribution than $\Phi$ given that $\UU_*$ is in~$\mathcal{O}_{n}$.  Defining, $\hat{\Phi}_{\leq r}~=~\begin{pmatrix} \hat{\Phi}_1^\top  \\ \vdots  \\ \hat{\Phi}_r^\top  \end{pmatrix} \in \R^{r\times r}$ and $\hat{\Phi}_{r <}~=~\begin{pmatrix} \hat{\Phi}_{r+1}^\top  \\ \vdots  \\ \hat{\Phi}_d^\top  \end{pmatrix} \in \R^{(d - r -1)\times r}$, it gives: 
    $$
    \kappa^{2}(\VV) \leq \ffrac{\sigma_{\max}^{2 \tilde{\alpha}} \lambda_{\max}(\hat{\Phi}_{\leq r}^\top \hat{\Phi}_{\leq r}) + \sigma_{r+1}^{2 \tilde{\alpha}} \lambda_{\max}(\hat{\Phi}_{r <} ^\top \hat{\Phi}_{r <})}{\sigma_{r}^{2 \tilde{\alpha}} \lambda_{\min}(\hat{\Phi}_{\leq r}^\top \hat{\Phi}_{\leq r})}  \,.
    $$
    Now, using the concentration inequality from \citet{chen2005condition} and \citet{vershynin_2012} (\Cref{app:prop:whishart_concentration_tail_ineq,app:prop:whishart_concentration_ineq}), we have the following three inequalities:
    \begin{align*}
        &\Proba{\kappa_{\max}(\hat{\Phi}_{\leq r}) > \frac{9 r^2}{\mathrm{p}^2}} < \mathrm{p} \,, \\
        &\Proba{\lambda_{\min}(\hat{\Phi}_{\leq r}^\top \hat{\Phi}_{\leq r}) \leq \frac{\mathrm{p}^2}{r}} \leq \mathrm{p} \,, \\
        & \mathbb{P}[\lambda_{\max}(\hat{\Phi}_{r <}^\top \hat{\Phi}_{r <}) \geq (\sqrt{d -r -1} + \sqrt{r} + \sqrt{2\log(2\mathrm{p}^{-1})}^2)] \\
        &\qqquad\leq \mathrm{p} \,,
    \end{align*}
    Next, as we have:
    \begin{align*}
        \Proba{A + \frac{B}{C} \leq a + \frac{b}{c}} &\geq \Proba{A \leq a \cap B \leq b \cap C \geq c} \\
        &= 1 - \Proba{A \geq a \cup B \geq b \cup C \leq c} \\
        &\geq 1 - \Proba{A \geq a} - \Proba{B \geq b} \\
        &\qquad- \Proba{C \leq c} \,,
    \end{align*}
    
    we deduce that for $\mathrm{p}$ in [0,1], with probability upper than $1 - 3\mathrm{p}$, we have:
    \begin{align*}
        \kappa(\VV)^2 &< \ffrac{r \bigpar{\sqrt{d -r -1} + \sqrt{r} + \sqrt{2\log(2\mathrm{p}^{-1})}}^2 \sigma_{r+1}^{2 \tilde{\alpha}}}{\mathrm{p}^2\sigma_{r}^{2 \tilde{\alpha}}} \\
        &\qquad +\frac{9 r^2 \sigma^{2 \tilde{\alpha}}_{\max}}{\mathrm{p}^2 \sigma^{2 \tilde{\alpha}}_{r} }\,.
    \end{align*}

    Next, using Jensen inequality for concave function, we have $\sqrt{d -r -1} + \sqrt{r} \leq \sqrt{2 (d - 1)} \leq \sqrt{2 d}$ and hence $( \sqrt{2d} + \sqrt{2\log(2\mathrm{p}^{-1})})^2 \leq 4 ( d + \log(2\mathrm{p}^{-1}) $,    
    which leads to $\kappa(\VV)^2 <\kappa_{\mathrm{p}}^2$ with:
    \begin{align*}
        \kappa_{\mathrm{p}}^2 := \frac{1}{\mathrm{p}^2} \bigpar{ 9r^2 \frac{ \sigma^{2 \tilde{\alpha}}_{\max}}{ \sigma^{2 \tilde{\alpha}}_{r} } +  4 r \bigpar{d +\log(2\mathrm{p}^{-1})} \ffrac{ \sigma_{r+1}^{2 \tilde{\alpha}}}{\sigma_{r}^{2 \tilde{\alpha}}} }   \,.
    \end{align*}

    We take $\mathrm{p} = 1/6$ and want to generate $m$ independent Gaussian matrix $(\Phi_{j})_{j=1}^m$ s.t. with probability $\mathrm{p}$, at least one results to have the condition number of $\VV$ lower than $\kappa_{1/6}$, therefore, we require to obtain $\sum_{k=1}^m \mathbb{1}_{\kappa^2(\VV_j) \leq \kappa_{1/6}^2} \geq 1$ with probability $\mathrm{P}$:
    \begin{align*}
        \mathrm{P} &= \Proba{\sum_{j=1}^m \mathbb{1}_{\kappa^2(\VV_j) \leq \kappa_{1/6}^2} \geq 1} \\
        &= 1 -\Proba{\sum_{j=1}^m \mathbb{1}_{\kappa^2(\VV_j) \leq \kappa_{1/6}^2} = 0} \\
        &\overset{(\Phi^{i,j})_{j=1}^m~\text{indep.}}{=} 1- \prod_{j=1}^m \Proba{\kappa^2(\VV_j) \geq \kappa_{1/6}^2}  \geq 1 - 1 / 2^m \,,
    \end{align*}
    and taking $m = - \log_2(1 - \mathrm{P})$ allows to conclude.
    Therefore, after $m$ different sampling of $\Phi$, we have at least one initialization such that with probability $\mathrm{P}$ we have $\kappa(\VV) < \kappa_{1/6}$ and it allows to conclude.
\end{proof}


\section{Proof of Subsection 3.3}
\label{app:sec:proof_exact_sol}

In this Section, we give the demonstrations of the results stated in \Cref{subsec:optimal_solution_analysis}. We start with the proof of \Cref{thm:upper_bound_with_proba}.

\begin{theorem}
\label{app:thm:upper_bound_with_proba}
    Let $r \leq d \wedge n$ in $\N^*$, using the power initialization (\Cref{algo:distributed_power_iteration}), for $\mathrm{p} \in ]0, 1[$, with probability at least $1 - 2\mathrm{p}$, we have:
    \begin{align*}
         &\min_{\UU \in \R^{n \times r} } \SqrdNrm{\SS - \UU \VV^\top}_\frob \leq \sum_{i >r} \sigma_i^2  \times \\
         &\qquad \bigpar{ 1 +  2 r \mathrm{p}^{-1}\bigpar{\ln(\mathrm{p}^{-2}) + \ln(2) r}\ffrac{(\sigma^2_{\max} - \sigma_i^2 ) }{\sigma_r^2} \frac{\sigma_i^{4 \alpha}}{\sigma_r^{4 \alpha}} }\,.
    \end{align*} 
\end{theorem}

\begin{proof}

    Let $\VV$ in $\R^{d\times r}$, then we have $\min_{\UU \in \R^{n \times r} } \SqrdNrm{\SS - \UU \VV^\top}_\frob = \SqrdNrm{\SS - \SS \VV (\VV^\top \VV)^{-1} \VV^\top}_\frob$.
    We define $\Tilde{\Phi} = \VV_*^\top \VV$ and denote $\PP = \tilde{\Phi} (\tilde{\Phi}^\top \tilde{\Phi})^{-1} \tilde{\Phi}^\top$ ($\PP$ is the projector on the subspace spanned by the columns of $\Tilde{\Phi}$). Then, we have $\SS \VV (\VV^\top \VV)^{-1} \VV^\top = \UU_* \SSigma \PP \VV_*$. We want to upper bound the following:
    \begin{align*}
        \SqrdNrm{\SS - \SS \VV (\VV^\top \VV)^{-1} \VV^\top}_\frob &= \sqrdnrm{\SSigma - \SSigma \PP}_\frob \\
        &= \sum_{i=1}^d \sigma_i^2 \sqrdnrm{(\PP - \Id_d)e_i}_2 \\
        &= \sum_{i=1}^d \sigma_i^2 (1 - \sqrdnrm{\PP e_i}_2) \,.
    \end{align*}
    
    But $\sqrdnrm{\PP e_i} = e_i^\top \PP \PP^\top e_i = e_i^\top \PP e_i = \PP_{ii} \in [0,1]$. Moreover, we have $\sum_{i = 1} ^d \PP_{ii} = \Tr{\PP} = r$, this implies:
    \begin{align}
    \label{app:eq:upper_bound_noise}
        \sqrdnrm{\SSigma - \SSigma \PP}_\frob &= \sum_{i=1}^d \sigma_i^2 (1 - \PP_{ii}) \nonumber\\
        &= \sum_{i >r} \sigma_i^2 (1 - \PP_{ii}) + \sum_{i \leq r} \sigma_i^2 (1 - \PP_{ii}) \nonumber  \\
        &\leq \sum_{i >r} \sigma_i^2 (1 - \PP_{ii}) + \sigma_{\max}^2\sum_{i \leq r}  (1 - \PP_{ii}) \nonumber \\
        &\overset{\sum_{i = 1}^d \PP_{ii} =r}{\leq} \sum_{i >r} \sigma_i^2 + \sum_{i >r} (\sigma^2_{\max} - \sigma_i^2 ) \PP_{ii}\,.
    \end{align}

    To compute $\PP_{ii}$ for any $i >r$, we write:
    \begin{align*}
        1 - \PP_{ii} &= \sqrdnrm{e_i - \PP e_i}_2 = \min_{y \in \mathrm{Im}(\PP)} \sqrdnrm{e_i - y}_2 \\
        &= \min_{y = \mathrm{Im}(\PP)} 1 - 2 e_i^\top y + \sqrdnrm{y}_2 \\
        &\overset{\mathrm{Im}(\PP) = \mathrm{Sp}(\tilde{\Phi})}{=} 1 - \max_{x \in \R^r} 2 e_i^\top \tilde{\Phi} x - \sqrdnrm{\tilde{\Phi} x}_2 \\
        &= 1 - \max_{x \in \R^r, \sqrdnrm{x}_2 = 1, \beta \in \R} 2 \beta e_i^\top \tilde{\Phi} x - \beta ^2\sqrdnrm{\tilde{\Phi} x}_2 \,,
    \end{align*}
    and we minimize it w.r.t. $\beta$, which gives $\PP_{ii} = \max_{x \in \R^r, \sqrdnrm{x}_2 = 1} \ffrac{(e_i^\top \tilde{\Phi} x)^2}{\sqrdnrm{\tilde{\Phi} x}_2}$. Now, because $\VV$ has been initialized using the power initialization (\Cref{algo:distributed_power_iteration}) and noting $\tilde{\alpha} = 2 \alpha + 1$, we have for the numerator:
    \begin{align*}
    (e_i^\top \tilde{\Phi} x)^2 = \sigma_i^{2 \tilde{\alpha}} (e_i^\top \UU_*^\top \Phi x)^2 = \sigma_i^{2 \tilde{\alpha}}  (\hat{\Phi}^\top_i x)^2\,,
    \end{align*}
    where $\hat{\Phi} := \UU_*^\top \Phi = \begin{pmatrix} \hat{\Phi}_1^\top  \\ \vdots  \\ \hat{\Phi}_d^\top  \end{pmatrix}$ has the same distribution than $\Phi$ given that $\UU_*$ is in $\mathcal{O}_n$. For the denominator, we do the same $\sqrdnrm{\tilde{\Phi} x}_2 = \sum_{j=1}^d \sigma_j^{2 \tilde{\alpha}} (\hat{\Phi}^\top_j x)^2 $. Therefore, we have:
    \begin{align*}
        \PP_{ii} &= \max_{x \in \R^r, \sqrdnrm{x}_2 = 1} \ffrac{\sigma_i^{2 \tilde{\alpha}}  (\hat{\Phi}^\top_i x)^2}{\sum_{j=1}^d \sigma_j^{2 \tilde{\alpha}} (\hat{\Phi}^\top_j x)^2 } \\
        &= \ffrac{1}{1 + \min_{x \in \R^r, \sqrdnrm{x}_2 = 1} \frac{\sum_{j=1, j\neq i}^d \sigma_j^{2 \tilde{\alpha}} (\hat{\Phi}^\top_j x)^2 }{\sigma_i^{2 \tilde{\alpha}}  (\hat{\Phi}^\top_i x)^2}} \\
        &\leq \ffrac{1}{1 + \frac{\sigma_r^{2 \tilde{\alpha}} }{\sigma_i^{2 \tilde{\alpha}}} \min_{x \in \R^r, \sqrdnrm{x}_2 = 1} \frac{\sum_{j \leq r} (\hat{\Phi}^\top_j x)^2 }{(\hat{\Phi}^\top_i x)^2}} \,.
    \end{align*}

    Next, we write:    
    \begin{align*}
        \min_{x \in \R^r, \sqrdnrm{x}_2 = 1} \frac{\sum_{j \leq r} (\hat{\Phi}^\top_j x)^2 }{(\hat{\Phi}^\top_i x)^2}  &= \min_{x \in \R^r, \sqrdnrm{x}_2 = 1} \ffrac{\sqrdnrm{ \hat{\Phi}_{\leq r} x}_2}{x^\top \hat{\Phi}_i \hat{\Phi}_i^\top x } \\
        &\geq \min_{x \in \R^r, \sqrdnrm{x}_2 = 1} \ffrac{\sqrdnrm{ \hat{\Phi}_{\leq r} x}_2}{\Tr{ \hat{\Phi}_i \hat{\Phi}_i^\top } } \\
        &\geq \frac{\lambda_{\min} (\hat{\Phi}_{\leq r}^\top \hat{\Phi}_{\leq r})}{ \sqrdnrm{\hat{\Phi}_i}_2 } \,,
    \end{align*} 
    
    where~$\hat{\Phi}_{\leq r}~=~\begin{pmatrix} \hat{\Phi}_1^\top  \\ \vdots  \\ \hat{\Phi}_r^\top  \end{pmatrix} \in \R^{r\times r}$, it follows:
    $$        \PP_{ii} \leq \ffrac{1}{1 + \frac{\sigma_r^{2 \tilde{\alpha}} }{\sigma_i^{2 \tilde{\alpha}}} \frac{\lambda_{\min} (\hat{\Phi}_{\leq r}^\top \hat{\Phi}_{\leq r})}{ \sqrdnrm{\hat{\Phi}_i}_2 }} \leq \ffrac{\sigma_i^{2 \tilde{\alpha}} }{\sigma_r^{2 \tilde{\alpha}}} \ffrac{ \sqrdnrm{\hat{\Phi}_i}_2 }{\lambda_{\min} (\hat{\Phi}_{\leq r}^\top \hat{\Phi}_{\leq r})}\,.$$

    To upper bound the fraction, we first use the concentration inequality from \citet{chen2005condition} (\Cref{app:prop:whishart_concentration_tail_ineq}) and we have: 
    $$
    \Proba{\lambda_{\min}(\hat{\Phi}_{\leq r}^\top \hat{\Phi}_{\leq r}) \leq \frac{\mathrm{p}^2}{r}} < \mathrm{p}\,.$$
    Second, we use the Chernoff inequality (\Cref{app:tech:prop:chernoff}) on $\sqrdnrm{\hat{\Phi}_i}_2 = \sum_{j=1}^r \hat{\Phi}_{ij}^2$ which is a $\chi_2(r)$-distribution, thus for any $t < 0$:
    \begin{align*}
        \Proba{\sqrdnrm{\hat{\Phi}_i}_2 \geq \delta_1} &\leq \fullexpec{\e^{tX}} \e^{-t\delta_1} = \prod_{j=1}^r \fullexpec{\e^{\hat{\Phi}^2_{ij}}} \e^{-t\delta_1} \,,
    \end{align*}
    and
    \begin{align*}
        \fullexpec{\e^{\hat{\Phi}^2_{ij}}} \e^{-t\delta_1} &\overset{\mathrm{(ii)}}{=} (1 - 2t)^{-r/2} \e^{-t\delta_1} \overset{\mathrm{s.t.}}{=} \mathrm{p} \,,
    \end{align*}
    where at (ii) we replace the expectation by the moment-generating function of the chi-square distribution, which is well defined if $t < 1/2$. Hence, taking $t=1/4$, we obtain $\Proba{\sqrdnrm{\hat{\Phi}_i}_2 \geq 2 (\ln(\mathrm{p}^{-1}) + \ln(2) r)} \leq \mathrm{p}  $.
    
    Next, as we have:
    \begin{align*}
        \Proba{\frac{A}{B} \leq \frac{a}{b}} &\geq \Proba{A \leq a \cap B \geq b} \\
        &= 1 - \Proba{A \geq a \cup B \leq b} \\
        &\geq 1 - \Proba{A \geq a} - \Proba{B \leq b}\,,
    \end{align*}
    
    we deduce that for $\mathrm{p}$ in $[0,1]$, with probability upper than $1 - 2\mathrm{p}$, we have:
    \begin{align*}
        \Proba{\ffrac{ \sqrdnrm{\hat{\Phi}_i}_2 }{\lambda_{\min} (\hat{\Phi}_{\leq r}^\top \hat{\Phi}_{\leq r})} \leq \ffrac{2 (\ln(\mathrm{p}^{-1}) + \ln(2) r)r}{\mathrm{p}^2}} > 1 - 2 \mathrm{p}
    \end{align*}
    
    And back to \Cref{app:eq:upper_bound_noise}, we have with probability upper than $1 - 2\mathrm{p}$:
    \begin{align*}
        &\sqrdnrm{\SSigma - \SSigma \PP}_\frob < \sum_{i >r} \sigma_i^2 \times \\
        &\qquad\bigpar{1 +   2r \mathrm{p}^{-2}\bigpar{\ln(\mathrm{p}^{-1}) + \ln(2) r}\ffrac{(\sigma^2_{\max} - \sigma_i^2 ) }{\sigma_r^2} \frac{\sigma_i^{4 \alpha}}{\sigma_r^{4 \alpha}}}\,,
    \end{align*}    

    which proves \Cref{thm:upper_bound_with_proba}.

\end{proof}

And now, we prove \Cref{cor:bound_with_r_star} which is derived from \Cref{thm:upper_bound_with_proba}.

\begin{corollary}
\label{app:cor:r_star_nb_of_samples}
    For any $ r_* \leq r \leq n \wedge d$, for $\alpha \in \N^*$, with probability $p \in ]0, 1[$, if we sample $m = \lfloor - \log_{2} (1-p) \rfloor$ independent matrices $(\Phi_{j})_{j=1}^m$ to form $\VV_j = \SS^\alpha \Phi$ and $\UU_j = \SS \VV_j (\VV_j \VV_j)^{-1} \VV_j$, at least one of the couple $(\UU_j, \VV_j)$ results in verifying: 
    \begin{align}
    \label{app:eq:upper_bound_with_proba_1_on_2}
        &\sqrdnrm{\SS - \UU_j \VV_j}_\frob < \sum_{i >r_*} \sigma_i^2 \times \\
        &\qqquad \bigpar{1 + 32 \ln(4) r_* (r_*+1) \ffrac{(\sigma^2_{\max} - \sigma_i^2 ) }{\sigma_{r_*}^2} \frac{\sigma_i^{4 \alpha}}{\sigma_{r_*}^{4 \alpha}}} \nonumber\,.
    \end{align} 
\end{corollary}

\begin{proof} 
    Let $r \in \llbracket r_*, n \wedge d \rrbracket$ and $\alpha \in \N^*$,  using the power initialization (\Cref{algo:distributed_power_iteration}), we have $\VV = (\SS^\top \SS)^\alpha \SS^\top \Phi$. As in the proof of \Cref{thm:bound_on_kappa,thm:upper_bound_with_proba}, we note $\Tilde{\Phi} := \VV_*^\top \VV$, with $\Phi = \begin{pmatrix} \Phi_1  \cdots  \Phi_r  \end{pmatrix}$  in $\R^{d \times r}$ s.t. all elements are Gaussian. Similarly, we define  $\Phi' = (\Phi_i \cdots \Phi_{r_*} )$ in $\R^{d \times r_*}$ the reduction of $\Phi$ to its $r_*$-first columns, this allows to define $\VV' = (\SS^\top \SS)^\alpha \SS^\top \Phi'$ and $\Tilde{\Phi}' = \VV_* \VV'$. We define $\PP, \PP'$ the corresponding projectors on the subspaces spanned by the column of $\Tilde{\Phi}$ and $\Tilde{\Phi}'$. Then we have $\mathrm{Im}(\PP') \subset \mathrm{Im}(\PP)$ or equivalently $\mathrm{Im}(\Id - \PP) \subset \mathrm{Im}(\Id - \PP')$, therefore:
    \begin{align*}
        \sqrdnrm{\SSigma( \PP - \Id_d)}_\frob &= \sum_{i=1}^n \sigma_i^2 \sqrdnrm{( \PP - \Id_d) e_i} \\
        &\overset{\text{Prop.~\ref{app:prop:proj_ineq}}}{\leq} \sum_{i=1}^n \sigma_i^2 \sqrdnrm{( \PP' - \Id_d) e_i} \\
        &\leq  \sqrdnrm{\SSigma( \PP' - \Id_d)}_\frob.
    \end{align*}
    
    Next, taking $\mathrm{p} = 1/4$ in \Cref{thm:upper_bound_with_proba}, it gives that using the power initialization strategy for any $r \in \llbracket r_*, n \wedge d \rrbracket$, we have with probability superior to $1/2$:
    \begin{align}
    \label{app:proof:upper_bound_condition_for_m_sample}
        &\sqrdnrm{\SSigma( \PP - \Id_d)}_\frob \leq \sqrdnrm{\SSigma( \PP' - \Id_d)}_\frob = \sqrdnrm{\SS - \UU \VV'^\top}_\frob \nonumber \\
        &\quad< \underbrace{\sum_{i >r} \sigma_i^2 \bigpar{1 +   32 \ln(4) r_* (r_*+1) \ffrac{(\sigma^2_{\max} - \sigma_i^2 ) }{\sigma_{r_*}^2} \frac{\sigma_i^{4 \alpha}}{\sigma_{r_*}^{4 \alpha}} }}_{:= \Box}\,.
    \end{align}

    We want to generate $m$ independent matrices $(\Phi_{j})_{j=1}^m$ to form $\VV_j = \SS^\alpha \Phi_j$ and $\UU_j = \SS \VV_j (\VV_j \VV_j)^{-1}$ s.t. with probability $\mathrm{P}$, at least one will results in verifying above \Cref{app:proof:upper_bound_condition_for_m_sample}, mathematically we require:
    \begin{align*}
        \mathrm{P} &:= \Proba{\sum_{j=1}^m \mathbb{1}_{\sqrdnrm{\SS - \UU_j \VV_j}_\frob \leq \Box} \geq 1} \\
        &= 1 -\Proba{\sum_{j=1}^m \mathbb{1}_{\sqrdnrm{\SS - \UU_j \VV_j}_\frob \leq \Box}
        = 0} \\
        &\overset{(\Phi_{j})_{j=1}^m~\text{indep.}}{=} 1- \prod_{j=1}^m \Proba{\sqrdnrm{\SS - \UU_j \VV_j}_\frob \geq \Box} \\
        &\geq 1 - 1/2^m \,,
    \end{align*}
    and taking $m  = -\log_{2} (1-\mathrm{P})$ allows to conclude.
\end{proof}


\section{A deeper description of the related works}
\label{app:sec:deeper_related_work}

In this Section, we recall the main results of \citet{halko2011finding,li2021communication}, the two main competitors of our works. We emphasize again that our results are not directly comparable as \citet{halko2011finding} obtain results on the $2$-norm in the general case $\alpha \in \N$, and \citet{li2021communication} provide a result solely on the $2$-norm distance between eigenspaces. We summarize the difference in \Cref{tab:related_work_power}.

\begin{table*}[t]
    \centering
    \caption{Comparison of our work with the \texttt{Distributed Randomized Power Iteration} \citep{halko2011finding} and with  \texttt{LocalPower} \citep{li2021communication}. $\mathscr{C}(\mathrm{SVD})$ corresponds to the computational cost of running a SVD (used for orthogonalization in most Python's libraries).}
    \label{tab:related_work_power}
    \resizebox{\textwidth}{!}{%
    \begin{tabular}{lccc}
    & \texttt{LocalPower} & \texttt{\makecell{Dist. Random. \\ Power Iter. }} & Our work \\
    \toprule \toprule
    Goal & \tcmv{\makecell{Approximate \\ eigenvector $\VV_*$}}& Matrix factorisation & Matrix factorisation\\
    \hline \hline \\
    Communic. cost & \makecell{$2Nd r \times (\tcmr{\alpha} +1)$, s.t.\\  $ \alpha = \Omega\bigpar{\ffrac{\sigma_{r_*} \log(d \epsilon^{-1})}{\tcmv{l}(\tcmr{\sigma_{r_*} - \sigma_{r_*+1}})} }$ } & \makecell{$2Nd r \times (\tcmr{\alpha} +1)$, s.t.\\  $ \alpha = \Omega\bigpar{\ffrac{\sigma_{r_*+1} \log(d)}{\epsilon}  }$ } & \makecell{$2Nd r \times (\tcmr{\alpha} +1)$, s.t.\\  $ \alpha = \Omega\bigpar{\ffrac{\log( \sigma_{\max} d \epsilon^{-1})}{\tcmv{\log(\sigma_{r_*}) - \log(\sigma_{r_*+1} })}} $ }   \\
    \hline \\
    Comput. cost (server) & $Ndr \times (\tcmr{\alpha}+1) + \tcmr{\mathscr{C}(\mathrm{SVD})}$& $Ndr\times (\tcmr{\alpha}+1)  + \tcmr{\mathscr{C}(\mathrm{SVD})}$ & $Ndr \times (\tcmr{\alpha}+1)$ \\
    \hline \\
    Comput. cost (clients) & $  \bigpar{4 n d r +  \tcmr{\mathscr{C}(\mathrm{SVD})}} \times \tcmr{\alpha l}$ & $4 n d r$ & \makecell{$\tcmo{\alpha} \times 2 n dr +\tcmo{T} \times 4  n dr $, s.t. \\  $T= \Omega \bigpar{\ffrac{\tcmr{\sigma_{\max}^{2\alpha}}}{\tcmv{\sigma_r^{2 \alpha}}} r^2 \log(\epsilon^{-1})}$}\\
    \hline \\
    Constraint on $r$ & \tcmr{$r = r_* + k$, $k>0$} & \tcmr{$r = r_* + k$, $k>0$} & \tcmv{$r \in \{0, d\}$}\\
    \hline \\
    Guarantee & $\|\VV \VV^\top - \VV_* \VV_*^\top \|_{\tcmo{2}} \leq \epsilon $ & $\E \|\SS - \QQ \QQ^\top \SS \|_{\tcmo{2}} \leq \epsilon_{\min} + \epsilon $ & $\sqrdnrm{\SS - \UU \VV^\top }_{\tcmo{\frob}} \leq \epsilon_{\min} + \epsilon $\\
    \hline \\
    Local algorithm & \tcmr{\makecell{Rely on a \\ SVD implementation}} & \tcmr{\makecell{Rely on a \\ SVD implementation}} & \tcmv{\makecell{Use GD and a ``restart \\strategy'' to improve $\kappa(\VV)$}} \\
    \hline \\
    Ext. to regularisation & \tcmr{\xmark} & \tcmr{\xmark} & \tcmv{\cmark} \\
    \hline\hline 
    \makecell{$\alpha = 0$ only if:} & \tcmr{\xmark} & \tcmv{$\sigma_{r_* + 1} \leq \ffrac{\epsilon}{\log(d)}$} & $\sigma_{r_* + 1} \leq \ffrac{\epsilon}{\sigma_{\max} d}$ \\
    \makecell{$\alpha \geq 1$, $\sigma_{r_*} = \Omega(1)$, \\ $\sigma_{r_*+1} = o(1)$ implies:} & $\epsilon = \Omega\bigpar{\ffrac{d}{\e^{\alpha l}}}$ & $\epsilon = \Omega\bigpar{\ffrac{\sigma_{r_*+1} \log(d)}{\alpha}}$ & \tcmv{$ \epsilon = \Omega \bigpar{\ffrac{\sigma_{\max} d}{\sigma_{r_*}^\alpha}}$} \\
    \bottomrule\bottomrule
    \end{tabular}
    }
\end{table*}

The next theorem comes from Corollary 10.10 of \citet{halko2011finding} which is the most detailed result on the $2$-norm in the case $\alpha \in \N$.

\begin{theorem*}[Corollary 10.10, from \citet{halko2011finding}]
    
Select a target rank $r_* \geq 2$ and an oversampling rank $r \geq r_* + 2$, with $r \leq \min \{n, d\}$. Execute \Cref{algo:distributed_power_iteration}, then:
\begin{align*}
&\FullExpec{\| \left(\mathbf{I}-\PP_{\VV}\right) \SS \|_2 }\leq \Bigg[\bigg(1+\sqrt{\frac{r_*}{r - r_*-1}}\bigg) \sigma_{r_*+1}^{2 \alpha+1}  \\
&\qqquad\qquad+ \frac{\mathrm{e} \sqrt{r}}{r - r_*}\bigg(\sum_{j>r_*} \sigma_j^{2(2 \alpha+1)}\bigg)^{1 / 2}\Bigg]^{1 /(2 \alpha+1)} \,,
\end{align*}
where $\PP_\VV$ is the projector on the columns of $\VV= (\SS^\top \SS)^\alpha \SS^\top$. Note that it is possible to obtain the result without expectation.
\end{theorem*}

We can make the following remarks.
\begin{itemize}[leftmargin=*]
    \item \textbf{Range of $r$.} Their theorem holds only for for $r \leq r_* + 2$.
    \item \textbf{Mutliplicative error term.} The error term is multiplied by $r_*$. In comparison, we obtain an enworsen rate depending on $r_*^2 \sigma^2_{\max} \sigma_i^{4 \alpha} / (\sigma_{r_*}^2 \sigma_{r_*}^{4 \alpha})$.
    \item \textbf{Dependence on $\alpha$.} However, their rate of convergence is worse w.r.t. $\alpha$. Indeed, the bound is asymptotically equivalent to $\FullExpec{\sqrdnrm{\left(\mathbf{I}-\PP_{\VV}\right) \SS }_2} \leq \sigma_{r_* + 1} (1 + \sqrt{d})^{1/(2\alpha +1)}$, therefore, doing a limited development for a constant $d$ and $\alpha$ tending to infinity, they obtain $\alpha = \Omega(\frac{\sigma_{r_* + 1} \log(d)}{\epsilon})$.
    On our side, from \Cref{cor:bound_with_r_star}, we obtain $ \alpha = \Omega\bigpar{\frac{\log( \sigma_{\max} d \epsilon^{-1})}{\tcmv{\log(\sigma_{r_*}) - \log(\sigma_{r_*+1} })}} $. In the regime $\sigma_{r_* + 1} \ll 1$, we both have $\alpha =0$, otherwise, the dependence on the precision $\epsilon$ is worse than our result.
    \item \textbf{Frobenius-norm.} \citet{halko2011finding} do not provide a bound on the Frobenius-norm in the case $\alpha \neq 0$, which we do.
\end{itemize}

Building upon the power method, \citet{li2021communication} have designed \texttt{LocalPower} which computes the top-$k$ singular vectors of $\SS$ using periodic weighted averaging. Below, we recall their main theorem.

\begin{theorem*}[Theorem 1 from \citet{li2021communication}]
After  $ \alpha = \Omega(\frac{\sigma_{r_*}}{l(\sigma_{r_*} - \sigma_{r_*+1}}) \log(d \epsilon^{-1}))$ step of communication, with probability $1 - \tau^{-\Omega(r-r_* +1)} - \e^{\Omega(d)}$, we obtain$\sqrdnrm{\VV \VV^\top - \VV_* \VV_*^\top }_2 \leq \epsilon $.
\end{theorem*}

We can make the following remarks.
\begin{itemize}[leftmargin=*]
    \item \textbf{Asymptoticity of  $\alpha$.} Authors give results only with an asymptotic $\alpha = \Omega\bigpar{\frac{\sigma_{r_*} \log(d \epsilon^{-1})}{\tcmv{l}(\tcmr{\sigma_{r_*} - \sigma_{r_*+1}})} }$  that can not be equal to zero in the regime where $\sigma_{r_* + 1} \ll 1$. On our side, we obtain $ \alpha = \Omega\bigpar{\frac{\log( \sigma_{\max} d \epsilon^{-1})}{\tcmv{\log(\sigma_{r_*}) - \log(\sigma_{r_*+1} })}} $ which allows $\alpha = 0$ in this regime. Note that in the regime $\sigma_{r_*} \approx \sigma_{r_* 1 }$, the two bounds are equivalent.
    \item \textbf{Approximating the right-side eigenvectors.} The theorem is only on the $2$-norm error of approximating the real left singular vector. In our paper, we are interested in the simplest goal of factorizing $\SS$ and providing a rate on the Frobenius norm.
\end{itemize}

\end{document}